\documentclass[sigconf]{acmart}
\AtBeginDocument{%
  }

\usepackage{empheq}
\usepackage{url}
\usepackage{algorithm}       % 伪代码浮动环境、标题、编号
\usepackage{algpseudocode}   % 伪代码语法关键字（if/for/while/function等）
\usepackage{multirow}
\usepackage{booktabs}
\usepackage{graphicx}
\usepackage{subcaption}
\usepackage{amsthm}

\newtheorem{theorem}{Theorem}[section]

\newtheorem{definition}[theorem]{Definition}
\newtheorem{assumption}[theorem]{Assumption}

\begin{document}

\title{LibraSpec: Dynamic Diffusion-Based Speculative Decoding via Marginal-Gain-Driven Optimization}


\author{Zexun Lin}
\affiliation{%
	\institution{Suzhou Institute for Advanced Research, University of Science and Technology of China}
	\city{Suzhou}
	\state{Jiangsu}
	\country{China}
}
\email{zexunlin@mail.ustc.edu.cn}

\author{Yuan Feng}
\affiliation{%
	\institution{Suzhou Institute for Advanced Research, University of Science and Technology of China}
	\city{Suzhou}
	\state{Jiangsu}
	\country{China}
}
\email{yfung@mail.ustc.edu.cn}

\author{Junlin Lv}
\affiliation{%
	\institution{Suzhou Institute for Advanced Research, University of Science and Technology of China}
	\city{Suzhou}
	\state{Jiangsu}
	\country{China}
}
\email{junlinlv@mail.ustc.edu.cn}

\author{Kevin S. Zhou}
\affiliation{%
	\institution{Suzhou Institute for Advanced Research, University of Science and Technology of China}
	\city{Suzhou}
	\state{Jiangsu}
	\country{China}
}
\email{skevinzhou@ustc.edu.cn}

\author{Xike Xie}
\correspondingauthor
\affiliation{%
\institution{Suzhou Institute for Advanced Research, University of Science and Technology of China}
\city{Suzhou}
\state{Jiangsu}
\country{China}
}
\email{xkxie@ustc.edu.cn}




\begin{abstract}
    Speculative decoding accelerates large language model inference by drafting multiple tokens for parallel verification, with efficiency critically determined by the speculative length selected at each decoding round.
    Existing dynamic speculation methods select the speculation length by estimating how many tokens will be accepted, which is reasonable for autoregressive drafters that generates tokens sequentially. The recent wave of diffusion-based drafters, however, generates candidate blocks in parallel at substantially lower drafting cost, shifting the key question from how many tokens to generate to how many generated tokens are worth verifying.
    %The speculative length at each decoding round in effect strikes a trade-off between the cost and the gain of speculation, and dynamic-length methods seek to optimize this trade-off on the fly. Designed for autoregressive draft models, existing methods predict how many tokens will be accepted so as to ration both draft generation and verification. Diffusion-based draft models overturn this premise: draft tokens now arrive in abundant, nearly free blocks, leaving only the question of how much of the draft is worth verifying. 
    We therefore reformulate dynamic speculative-length selection as expected-speedup optimization and derive a marginal criterion that extends the speculative sequence only when its acceptance gain outweighs the additional verification cost.
    %We derive a marginal criterion under which a length adjustment improves expected speedup only when its acceptance gain justifies the additional verification cost.{\color{red}This renders dynamic length control substantially more demanding — it calls for explicitly weighing the verification cost against the acceptance gain, rather than merely judging whether each token is likely to be accepted. We model this trade-off from the perspective of marginal gain: an additional draft position is worth verifying only if its marginal gain per unit of verification cost exceeds the speedup already achieved.} 
    Building on this criterion, we develop \textit{LibraSpec}, 
    a training-free and plug-and-play algorithm that iteratively determines the speculative length using drafter confidence scores.
    %an iterative length-optimization algorithm built around \textit{the single beneficial adjustment} — a speculative length adjustment that provably improves the expected speedup. 
    Theoretically, we prove that LibraSpec monotonically converges toward the optimal speculative length. %without ever predicting it. 
    %We show that LibraSpec is a training-free, plug-and-play algorithm, integrating seamlessly with existing diffusion-based speculative decoding methods. 
    Experiments across six target models, three diffusion-based speculative decoding methods, and math, coding, and chat benchmarks show consistent improvements under both greedy and sampling settings, achieving a further $0.5\sim1.5\times$ improvement over baselines and up to $8.49\times$ speedup over autoregressive decoding. 
    %yielding an additional $0.5\sim1.5\times$ end-to-end speedup.
\end{abstract}

\maketitle

\section{Introduction}

Autoregressive large language models (LLMs) generate tokens sequentially, making decoding a major bottleneck in long-context reasoning, including code generation and interactive applications \citep{gpt5-openai,efficient-tay}. 
%have demonstrated expert-level performance on a wide range of long-context tasks, including complex reasoning and code generation \citep{gpt5-openai}. Nevertheless, their inherently sequential token-by-token generation process introduces a critical inference bottleneck, resulting in limited decoding throughput, memory-bound and increased latency \citep{efficient-tay}. To alleviate this issue, 
Speculative Decoding \citep{sd-leviathan,medusa-cai,eagle-li,eagle2-li,eagle3-li,dflash-chen, draft-verify-zhang, longspec-yang} alleviates this bottleneck with lossless acceleration with formal guarantees. %Specifically, it operates in three stages: a lightweight draft model first generates a sequence of candidate tokens, the target model then verifies these candidates in parallel, and finally, multiple valid tokens are accepted simultaneously based on a strict acceptance criterion.
At each decoding round, a lightweight draft model first proposes a sequence of candidate tokens, which are then verified in parallel by the target model; multiple valid tokens can consequently be accepted at once. The resulting speedup depends critically on the \emph{speculative length}, i.e., the number of candidate tokens submitted for verification. A short length underutilizes parallel verification, whereas an excessively long length wastes target-model computation when an early rejection invalidates the remaining suffix \citep{disco-mamou,ltd-zhang,blockpilot-zhang}. 
%Selecting the speculative length is therefore an expected-speedup optimization problem that balances the acceptance gain of additional draft tokens against their drafting and verification costs.
Because token predictability varies across decoding rounds, the speedup-maximizing length also changes throughout generation. This motivates \emph{dynamic speculative-length selection}: at each round, the system determines the length that best balances the acceptance gain of additional draft tokens against their computational cost.

\begin{figure*}[thbp]
	\centering
	\includegraphics[width=0.99\textwidth]{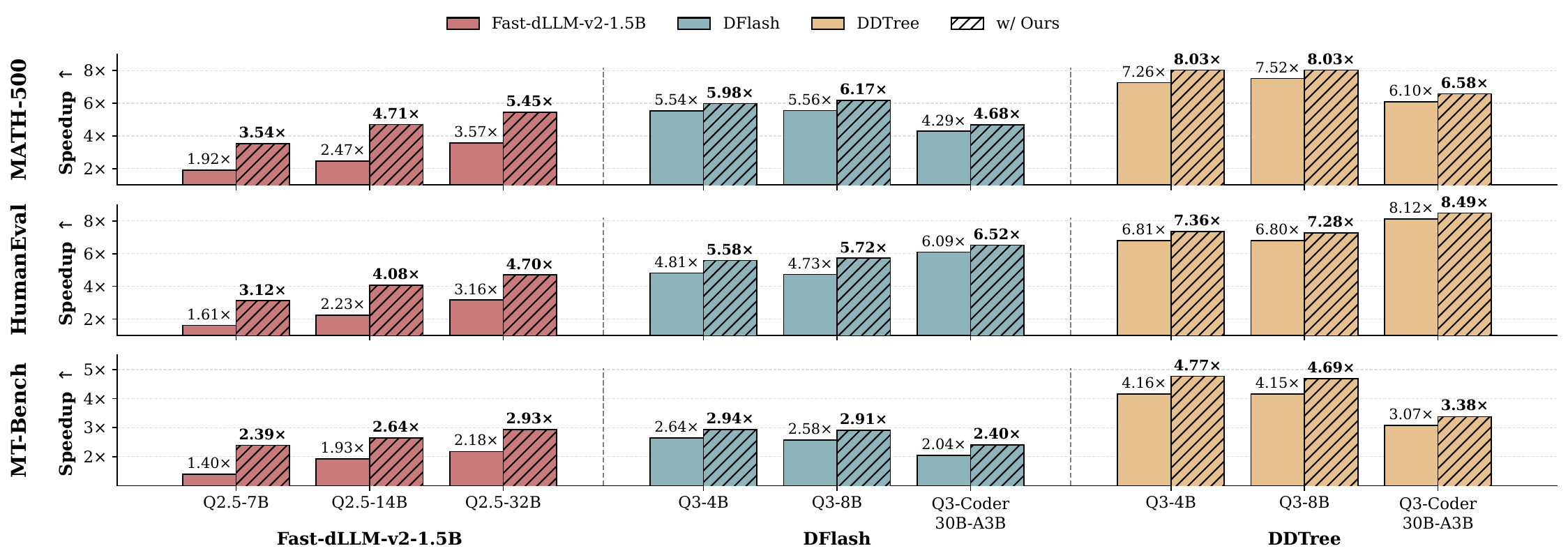}
	\caption{Speedup gains achieved by LibraSpec when integrated to different methods (see experiments for other benchmarks).}
	\label{fig:three_vertical}
\end{figure*}

%The end-to-end speedup of this mechanism is governed by a single critical control variable: the speculative length chosen at each decoding round. In effect, this length strikes a trade-off between the cost and the gain of speculation. Speculating too few tokens leads to "under-speculation," leaving the acceleration potential of parallel verification unexploited; speculating too many causes "over-speculation," where an early rejection instantly invalidates all subsequent draft tokens and the verification effort spent on them is wasted \citep{disco-mamou, ltd-zhang, blockpilot-zhang}. Crucially, the optimal point of this trade-off is not fixed — it fluctuates substantially across decoding steps as the difficulty of upcoming tokens varies (Figure \ref{fig:sl-comparison}), so any static choice is doomed to alternate between the two failure modes. A growing body of work therefore adjusts the speculative length dynamically during inference. Despite their diversity, these methods follow a common recipe: predicting the optimal length in one shot at each round, either by training an auxiliary predictor to forecast the accepted length , or by thresholding heuristic signals such as historical acceptance rates, confidence scores, and entropy. This paradigm was designed for the autoregressive regime in which it originated: by forecasting the number of tokens likely to be accepted, it imposes quota control over both draft generation and verification, economizing the expenditure of sequential, costly drafting along with the verification spent on tokens doomed to rejection.

Existing dynamic-length methods select the speculative length by predicting the accepted length, which is well-suited to autoregressive drafters because each additional candidate incurs sequential drafting cost. The recent wave of diffusion-based drafters, however, is reshaping speculative decoding by generating candidate blocks in parallel at substantially lower drafting cost \citep{dart-liu,dflash-chen}. This changes the optimization target: the key question shifts from how many tokens to generate to how many generated tokens are worth verifying. Consequently, estimating the expected accepted length alone is insufficient to determine whether extending the speculative prefix will improve speedup, as even a likely accepted extension may provide insufficient gain to justify its additional verification cost \citep{sssd-michele}. % {\color{red} @lin, citations or references to revised fig 1}

\begin{figure}[htbp]
    \centering
    \includegraphics[width=0.99\linewidth]{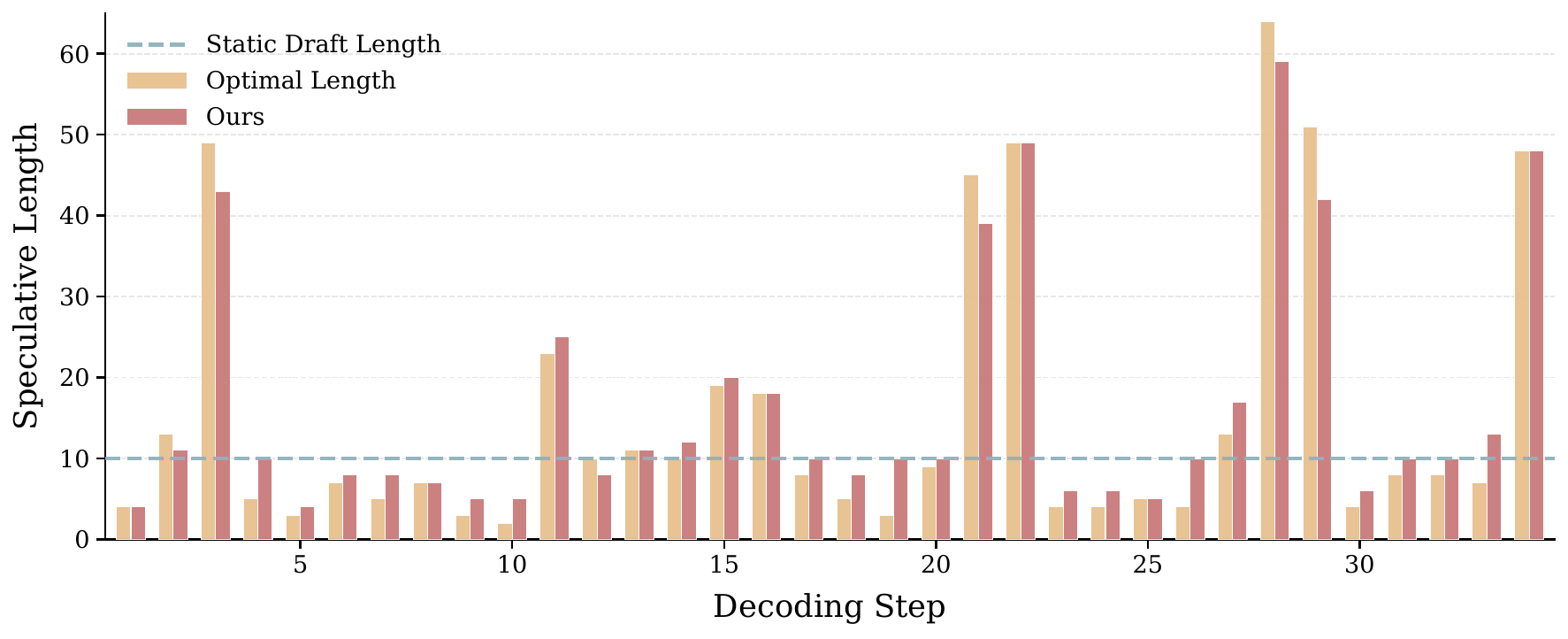}
    \caption{Static speculative length vs.\ LibraSpec vs.\ optimal speculative length, using the FastdLLM drafter with Qwen2.5-7B as the target model on MATH-500.}
    \label{fig:sl-comparison}
\end{figure}

%The emergence of diffusion-based draft models, however, overturns this premise. Block diffusion models decode entire blocks of masked tokens in parallel, rendering draft tokens abundant and virtually free \citep{dart-liu,dflash-chen}. 
%Half of the quota's purpose thus evaporates: drafting no longer needs rationing, and the decision that remains concerns verification alone — with a wealth of draft tokens already in hand, one must judge, position by position, whether admitting one more token into verification would raise the speedup beyond what is already attained. Yet the inherited paradigm, which only predicts whether draft tokens will be accepted, never weighs the cost of verifying them against the speedup they bring — a token that is likely to be accepted can still dilute the overall speedup.

%To address this mismatch, we formulate explicitly model the trade-off from the theoretical perspective of marginal gain and recast the problem as a marginal-gain-driven optimization. 
To address this mismatch, we reformulate dynamic speculative-length selection from an accepted-length prediction problem into an expected-speedup optimization problem.
Concretely, extending the speculative sequence improves expected speedup only when the added segment provides a higher acceptance-gain-to-verification-cost ratio than the current sequence. Although the globally optimal speculative length is not directly available in closed form, whether a local length adjustment improves expected speedup can be evaluated under our model. We formalize such a speedup-improving update as a \emph{beneficial adjustment} and derive necessary and sufficient conditions for determining its maximum admissible length. 
We further prove that expected speedup is unimodal with respect to the speculative length.
This structure recasts dynamic-length selection as iterative optimization: under our speedup model, successive beneficial adjustments monotonically improve expected speedup and converge in finitely many steps to a globally optimal speculative-length interval.

Based on this analysis, we propose \textsc{LibraSpec}, a training-free and plug-and-play algorithm for diffusion-based speculative decoding.
%We instantiate this scheme as LibraSpec, an iterative length-optimization algorithm built around the single beneficial adjustment and running entirely at inference time. 
At each decoding round, \textsc{LibraSpec} uses confidence scores produced by the diffusion drafter for the probabilities unavailable before verification. 
It then iteratively applies the marginal criterion to determine how many candidate tokens should be submitted for verification.
%estimates already available during drafting, and iterates to determine the final speculative length. 
This design requires neither an auxiliary predictor nor additional training and can be directly integrated with existing diffusion-based speculative decoding methods.
As illustrated in Figure~\ref{fig:sl-comparison}, \textsc{LibraSpec} closely tracks the oracle optimal length obtained by exhaustive search, while adapting to changes across decoding rounds.
%significantly outperforming any static speculative length.  
%As a training-free, plug-and-play strategy, LibraSpec integrates seamlessly with existing diffusion-based speculative decoding frameworks. 

We evaluate \textsc{LibraSpec} with six LLMs and three state-of-the-art (SOTA) diffusion-based speculative decoding methods.
Extensive experiments on mathematical reasoning, code generation, and chat benchmarks demonstrate consistent improvements brought in by \textsc{LibraSpec} 
%consistently accelerates Fast-dLLM, DFlash, and DDTree 
under both greedy and sampling settings. 
It further improves the end-to-end speedup by $0.5\sim1.5\times$ over diffusion-based methods and achieves up to $8.49\times$ speedup over vanilla autoregression-based methods. %end-to-end speedup on top of these strong baselines 
The main contributions are summarized as follows.

\begin{itemize}
    \item We identify that the emerging shift toward diffusion-based drafters changes the objective of dynamic speculative-length selection from accepted-length estimation to direct optimization of the acceptance-gain--verification-cost trade-off.    

    \item We formulate the problem as expected-speedup optimization, derive a marginal criterion for beneficial length adjustments, and derive necessary and sufficient conditions identifying the maximum admissible adjustment that guarantees improvement.

    \item We propose \textsc{LibraSpec}, a training-free and plug-and-play dynamic-length algorithm that consistently accelerates multiple diffusion-based speculative decoding methods across models, tasks, and decoding settings.

\end{itemize}

\section{Related work}

\subsection{Diffusion-Based Speculative Decoding}

Since Speculative Diffusion Decoding \citep{sdd-Christopher} first introduced discrete diffusion models as speculative drafters, several follow-up studies have been proposed. Methods \citep{specdiff-pan, specdiff-2-sandler, llmknow-samragh} represented by DiffuSpec \citep{diffuspec-li} employ large-scale pretrained diffusion language models as training-free speculative drafters, leveraging inference-time search strategies or train–test alignment techniques to improve draft quality and acceptance rates. However, these methods depend on large-scale draft models, which incur significant memory overhead and inference latency. In contrast, DFlash \cite{dflash-chen} trains a lightweight block-wise diffusion drafter that exploits KV Injection to extract rich contextual features from the target model, allowing an entire draft block to be generated in a single forward pass. Likewise, methods \citep{dart-liu, taps-wang, jetspec-hu} represented by DDTree \citep{ddtree-ringel}, build draft trees based on the token distributions obtained from a single diffusion forward pass, thereby further improving the accepted draft length. 
Our method can be seamlessly integrated with these diffusion-based speculative drafters. In this work, we demonstrate its effectiveness when combined with representative methods, including FastdLLM, DFlash and DDTree, consistently achieving additional decoding speedups across diverse benchmarks.

\subsection{Dynamic Speculative Decoding}

Since the speculative block length plays a critical role in determining the overall decoding efficiency, a growing body of work has focused on dynamically adjusting the draft length during inference.
Existing methods are mainly divided into two categories. Methods \citep{specdec-huang, ltd-zhang, blockpilot-zhang, spec-verify-kim} represented by DISCO \citep{disco-mamou} train a predictor to forecast the length of draft generation, which achieves relatively accurate prediction of the expected received length. However, they require additional training and rely heavily on the generalization ability of the predictor.
In contrast, methods \citep{del-zarch, adadecode-wei, banditspec-hou} such as Gemma 4 MTP \citep{gemma4-mtp} and FailFast \citep{failfast-pan} leverage heuristic signals, including historical acceptance rates, confidence scores, and entropy, as indirect indicators to guide the expected future acceptance length. While these approaches incur no additional training cost, they fail to account for the computational overhead introduced by the verification stage.
More importantly, these methods inherit the optimization paradigm of autoregressive speculative decoding, where the objective is to estimate the expected accepted prefix length. In contrast, LibraSpec explicitly models the trade-off between the expected benefit of successful speculation and the verification overhead, providing a novel perspective for dynamic speculative decoding.

\section{Methodology}

% NOTE: 介绍符号系统
\subsection{Preliminaries}
\label{sc:preliminaries}

Speculative decoding accelerates autoregressive generation by combining a lightweight draft model $\mathcal{M}_{\mathrm{draft}}$ with a target model $\mathcal{M}_{\mathrm{target}}$ \citep{sd-leviathan}. 
Given a prefix $x_{\le t}$, the draft model generates a speculative continuation $(x_{t+1}, \ldots, x_{t+d})$, producing a probability distribution $q_v = P_{\mathcal{M}_{\mathrm{draft}}}(x_{t+v}\mid x_{<t+v})$ for each drafted position $v\in\{1,\ldots,d\}$. 
The target model then evaluates all drafted positions in parallel via a single forward pass, yielding the target distribution $p_v = P_{\mathcal{M}_{\mathrm{target}}}(x_{t+v}\mid x_{<t+v})$. A subsequent verification step compares $q_v$ and $p_v$ to determine whether each drafted token should be accepted or corrected. 
Consequently, this speculative mechanism effectively bypasses the sequential bottleneck of standard autoregressive decoding without altering target output distribution of $\mathcal{M}_{\mathrm{target}}$.

\subsection{How Speculative Length Affects Diffusion-Based Speculative Speedup}
\label{sc:speedup}

To quantify the speedup of diffusion-based speculative decoding, consider a prefix $x_{\le t}$ and a length-$d$ draft block $\mathcal{D}_{1:d}=(x_{t+1},\ldots,x_{t+d})$. Let $T_d^{\mathrm{draft}}$ and $T_d^{\mathrm{verify}}$ denote the time required to generate and verify this block. Assuming the target model accepts $\tau_d$ consecutive tokens and has a standard per-token latency of $L_d^{\mathrm{target}}$, the overall speculative decoding speedup, $\eta_d$, can present as:

\begin{equation}
    \label{eq:speedup-objective}
    \eta_d = \frac{L_d^{\mathrm{target}}}{L_d^{\mathrm{spec}}} = \frac{\tau_d\,L_d^{\mathrm{target}}}{T_d^{\mathrm{draft}} + T_d^{\mathrm{verify}}}
\end{equation}

where $L_d^{\mathrm{spec}} = \frac{T_d^{\mathrm{draft}}+T_d^{\mathrm{verify}}}{\tau_d}$ is the average per-token latency of speculative decoding. 
    
In practice, as shown in Fig.~\ref{fig:draft-verify-rate}, compared with autoregressive drafters, diffusion-based drafters reduce the drafting overhead of speculative decoding by several folds, rendering the drafting cost negligible relative to the target-model inference cost.

\begin{figure}[htbp]
    \centering
    \includegraphics[width=\linewidth]{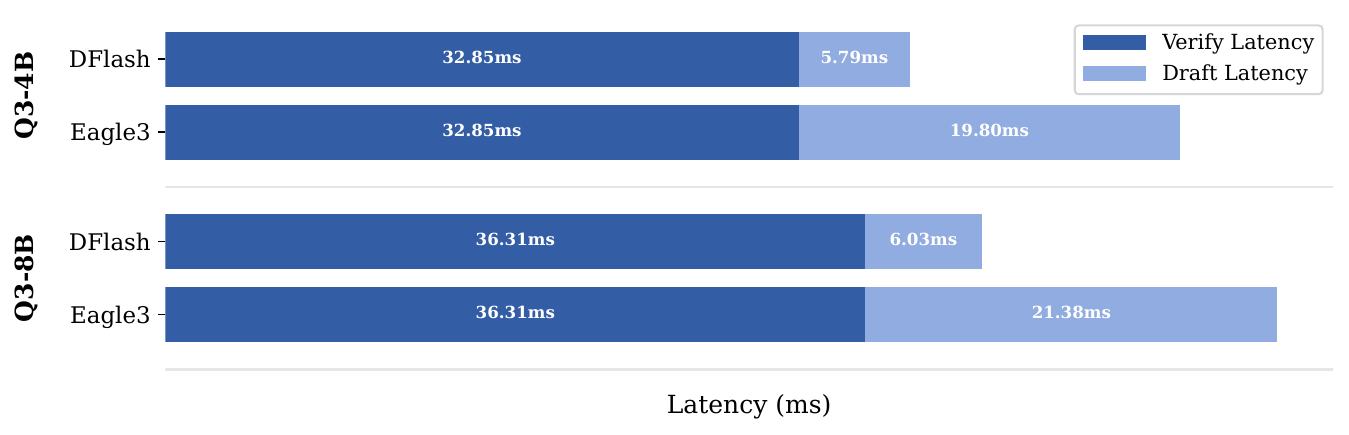}
    \caption{Latency breakdown at a speculative length of 16 tokens, measured on an A100 GPU with 1K context length.}
    \label{fig:draft-verify-rate}
\end{figure}

Therefore, to analyze how varying the speculative length affects overall efficiency, we consider a simplified setting where an infinitely long draft sequence is already available, allowing us to freely adjust the speculative length.

\begin{equation}
    \label{eq:simple-speedup-objective}
    \eta_d \approx \frac{\tau_d L_d^{\mathrm{target}}}{T_d^{\mathrm{verify}}}.
\end{equation}

In this setting, the drafting cost is fixed and independent of the speculative length, leaving the verification stage as the only component affected by the choice of speculative length. Consequently, the key question becomes how the speculative length should be adjusted to boost the decoding speedup.

Directly optimizing this objective is intractable. Therefore, we instead iteratively adjust the speculative length, allowing the optimization process to progressively converge toward the optimal speculative length. We require every speculative-length adjustment to satisfy the conditions of \emph{a single beneficial adjustment}, as formalized in Definition \ref{def:speedup}.

\begin{center}
\colorbox{lightgray!20}{\parbox{0.98\linewidth}{

\begin{definition}[A Single Beneficial Adjustment]
	\label{def:speedup}
    We say that a speculative length adjustment from $d$ to $d'$ is a single beneficial adjustment if it satisfies the following conditions.

    \begin{equation}
        \mathbb E[\eta_{d'}] > \mathbb E[\eta_d]
    \end{equation}
\end{definition}
}}	
% \vspace{-0.3cm}
\end{center}

\subsection{Achieving a Single Beneficial Adjustment of Speculative Length}

Definition~\ref{def:speedup} defines the speedup criterion that a single beneficial adjustment from $d$ to $d'$ must satisfy. As shown in Eq.~\ref{eq:simple-speedup-objective}, decoding speedup is jointly determined by the accepted draft length and the verification latency. Therefore, a single beneficial adjustment fundamentally corresponds to achieving a more favorable trade-off between the expected acceptance gain and the additional verification cost. This trade-off can be characterized by quantifying the marginal gain and the marginal cost introduced by the adjustment, as formalized in Theorem~\ref{thm:gain-cost}.

% By Definition~\ref{def:speedup}, a beneficial adjustment must improve expected speedup. Since Eq.\ref{eq:simple-speedup-objective} depends jointly on accepted speculative length and verification latency, this requires a favorable marginal acceptance-gain–verification-cost trade-off, formalized in Theorem\ref{thm:gain-cost}.

\begin{theorem}
    \label{thm:gain-cost}

    A speculative length adjustment achieves a single beneficial adjustment if and only if the marginal benefit of the adjusted draft segment is more favorable than the current average benefit.
    
    \begin{equation}
        \begin{cases}
            \displaystyle
            \frac{\mathbb{E}[\tau_{d:d'}]}{T^{\mathrm{verify}}_{d:d'}}
            >
            \frac{\mathbb{E}[\tau_d]}{T^{\mathrm{verify}}_d},
            & d<d', \\[1.5em]
            \displaystyle
            \frac{\mathbb{E}[\tau_{d':d}]}{T^{\mathrm{verify}}_{d':d}}
            <
            \frac{\mathbb{E}[\tau_d]}{T^{\mathrm{verify}}_d},
            & d>d'.
        \end{cases}
    \end{equation}

    where the first case ($d<d'$) corresponds to extending the draft sequence, requiring the marginal benefit of the appended draft segment to exceed the current average benefit, while the second case ($d>d'$) corresponds to truncating the draft sequence, requiring the marginal benefit of the removed suffix to be lower than the current average benefit.
\end{theorem}

\begin{proof}
	See Appendix \ref{proof:gain-cost} for details.
\end{proof}

Theorem~\ref{thm:gain-cost} provides a criterion for determining whether a speculative length adjustment is beneficial. Under prefix acceptance, however, the contribution of later speculative tokens depends on whether all preceding tokens are accepted. As a result, different draft positions contribute differently to the overall gain-cost trade-off. In particular, we strengthen the criterion of a single beneficial adjustment by requiring that the marginal benefit of the speculative suffix starting from position i exceeds the average benefit accumulated up to that position.

By expanding the expressions for the marginal gain, $\mathbb E[\tau_{i:d'}] = \sum_{j=i}^{d'}\prod_{k=i}^{j}p_k$, and the marginal verification cost $T_{i:d'}^{\mathrm{verify}}=c\,(d'-i)$, where $c$ denotes the average verification cost per speculative position, we obtain the corresponding necessary condition for a speculative length adjustment to achieve a single beneficial adjustment, as shown in Theorem \ref{thm:necessary-condition}.

% 必要性
\begin{theorem}
	\label{thm:necessary-condition} 
    Under the strengthened position-wise criterion, A Single Beneficial Adjustment from $d$ to $d'$ \textbf{only if} it satisfies:

    \begin{equation}
        \begin{aligned}
            d' &\le \epsilon_i, \quad \forall\, i \in [1, d], \\
            \epsilon_i &= \max\left\{d' \in \mathbb Z \mid d' < \frac{T_i^{\mathrm{verify}}}{\tau_i\,c} \, \sum_{j=i}^{d'}\prod_{k=i}^{j}p_k + i\right\}.
        \end{aligned}
    \end{equation}
\end{theorem}

\begin{proof}
	See Appendix \ref{proof:necessary-condition} for details.
\end{proof}

The preceding theorem establishes the position-wise bounds that every beneficial adjustment must satisfy. Necessity alone, however, does not establish that an adjusted length satisfying these bounds actually improves the decoding speedup.
We therefore prove sufficiency by exploiting the monotonic structure induced by prefix acceptance. As the candidate length increases, each newly added token provides a non-increasing conditional acceptance contribution, and hence the average marginal gain associated with each draft position is also non-increasing. Consequently, the lengths satisfying each position-wise constraint form a contiguous feasible interval. Taking the intersection of these intervals shows that satisfying all position-wise bounds jointly is sufficient for a single beneficial adjustment. A detailed proof is provided in Appendix \ref{proof:sufficient-condition}.

% 充分性
\begin{theorem}
	\label{thm:sufficient-condition}
    A Single Beneficial Adjustment from $d$ to $d'$ \textbf{if} it satisfies:

    \begin{equation}
        \begin{aligned}
            d' &= \min_i \epsilon_i, \quad \forall\, i \in [1, d], \\
            \epsilon_i &= \max\left\{d' \in \mathbb Z \mid d' < \frac{T_i^{\mathrm{verify}}}{\tau_i\,c} \, \sum_{j=i}^{d'}\prod_{k=i}^{j}p_k + i\right\}.
        \end{aligned}
    \end{equation}
\end{theorem}

Theorem~\ref{thm:sufficient-condition} provides a sufficient condition for achieving a single beneficial adjustment. Any speculative length exceeding this bound inevitably violates the necessary condition for achieving a single beneficial adjustment, while choosing a smaller length leads to conservative updates that unnecessarily sacrifice expected acceptance gains without guaranteeing additional speedup improvements.

\begin{algorithm}[tbp]
	\caption{LibraSpec}
	\label{alg:LibraSpec}
	\begin{algorithmic}[1]
		\Require Block size $n$, Maximum Speculation Length $d_{\text{max}}$
		\Ensure Draft tokens
        \State Generate one draft block of size $n$.
		\State $d \leftarrow n$; $B \leftarrow 0$
		\While{$d < d_{\text{max}}$}
        
    		\State $\epsilon^{*} \leftarrow \min\limits_{i} \epsilon_i$ where $\epsilon_i = \lfloor\alpha \, \sum_{j=i}^{d}\prod_{k=i}^{j}q_k + i\rfloor$, $\forall i = 1,\dots,d$
    		
    		\If{$\epsilon^{*} \le d$}     % 若缩减，直接退出
                \State $d \leftarrow \epsilon^{*}$; break
    		\EndIf
            
    		\State $B \leftarrow B - (\epsilon^{*} - d)$; $d \leftarrow \epsilon^{*}$     % 先从Budget中拿

            \If{$B < 0$}      % Budget不够就再生成
                \State Generate $\left\lceil\frac{-B}{n}\right\rceil$ draft blocks;
$B \leftarrow B + \left\lceil\frac{-B}{n}\right\rceil n$.
            \EndIf
		\EndWhile
		
		\State Submit $d$ tokens for verification.
	\end{algorithmic}
\end{algorithm}

\subsection{Determining the Final Speculative Length via Iterative Adjustments}

Theorem~\ref{thm:sufficient-condition} provides a theoretically grounded criterion for obtaining a single beneficial adjustment. Repeatedly applying this criterion progressively refines the speculative length toward the optimal decoding speedup. An immediate question is whether repeatedly applying beneficial adjustments indeed converges to an optimal speculative length. 

To answer this question, we begin by analyzing the marginal contribution of a single draft token. Under prefix acceptance, the conditional marginal benefit of each additional token is non-increasing, whereas the marginal verification cost remains approximately constant. Therefore, the marginal efficiency of extending the speculative block decreases as the speculative length grows. Once this marginal efficiency no longer exceeds the current average efficiency, subsequent extensions cannot restore an improvement. We therefore establish the unimodality of the expected decoding speedup with respect to the speculative length, as shown in \ref{thm:unimodal}.

\begin{theorem}
    \label{thm:unimodal}
    The expected speedup is a unimodal function of the speculative length.
    \begin{equation}
        \frac{\mathbb{E}[\tau_{d+1}]}{T^{\mathrm{verify}}_{d+1}}
        >
        \frac{\mathbb{E}[\tau_d]}{T^{\mathrm{verify}}_d} 
        \Longleftrightarrow
        \prod_{i = 1}^{d+1}p_i
        >
        c\,\frac{\mathbb{E}[\tau_d]}{T^{\mathrm{verify}}_d} 
    \end{equation}
\end{theorem}

\begin{proof}
    See Appendix \ref{proof:unimodal} for details.
\end{proof}

The unimodal structure implies that every beneficial adjustment moves the speculative length toward the unique optimum. We therefore establish the convergence of the iterative adjustment procedure, as shown in \ref{thm:convergence}.

\begin{theorem}
    \label{thm:convergence}
    By iteratively applying single beneficial adjustments, the speculative length converges in finitely many steps to $d'$, which lies in the globally optimal interval of speculative lengths.

    \begin{equation}
        \mathbb E[\eta_{d'}]
        =
        \max_{d\in\mathbb Z_{>0}}\mathbb E[\eta_{d}].
    \end{equation}
\end{theorem}

\begin{proof}
    See Appendix \ref{proof:convergence} for details.
\end{proof}

To apply the theoretical result in practice, two practical challenges must be addressed. First, the target-model acceptance probabilities $p_i$ are unavailable before verification. Second, the idealized assumption of an infinitely long draft sequence does not hold in practice. Since draft tokens are generated on demand, the confidence scores of future draft tokens are unavailable before they are generated. 

\begin{assumption}
    \label{ass:specforge}
    The draft-model confidence scores $q_i$ are well calibrated with the corresponding target-model acceptance probabilities $p_i$ \citep{specforge-li, making-pan}.
\end{assumption}

To address the first challenge, under this assumption, the draft-model confidence scores can be used as a proxy for the unknown target-model acceptance probabilities. However, this calibration generally deteriorates as the speculative horizon increases, and the rate of deterioration varies across draft models. We therefore constrain online speculation with a model-dependent maximum speculative length, $d_{\mathrm{max}}$. The influence of $d_{\mathrm{max}}$ is further analyzed in Section~\ref{sc:maxd}. To address the second challenge, as a practical compromise, we compute the adjustment criterion using only the confidence scores of the draft tokens generated so far in the current draft block. Since $T_i^{\mathrm{verify}}$, $\tau_i$, and $c$ are runtime-dependent, we absorb them into a single trade-off hyperparameter $\alpha=T_i^{\mathrm{verify}}\tau_i^{-1}c^{-1}$, which controls the aggressiveness of speculative length adjustment.

Building upon the above theoretical derivations, we propose a plug-and-play dynamic speculative decoding strategy, LibraSpec (Algorithm~\ref{alg:LibraSpec}), which dynamically adjusts the speculative block length through iterative refinements. At each iteration, LibraSpec seeks a single beneficial adjustment, thereby progressively improving the theoretical decoding speedup.

% 解释一下在具体实现过程中为什么要“先生成草稿块，后根据结果消耗草稿token”
To avoid frequent calls to the draft model caused by continual adjustments of the speculative length, LibraSpec maintains a draft budget $B$. At the beginning of each speculative round, an initial draft block is generated, and subsequent length extensions first consume the remaining budget before requesting additional draft blocks. New draft blocks are generated only when the budget is exhausted, reducing unnecessary draft-model calls and the associated I/O overhead.

% To avoid frequent draft model calls caused by iterative adjustments of the speculative length, LibraSpec employs a budget-based draft generation mechanism. At the beginning of each speculative round, it generates a fixed-size draft block and initializes a draft budget $B$ representing the available draft tokens. Subsequent length extensions first consume the remaining budget; additional draft blocks are generated only when the budget is insufficient. By batching draft generation in this manner, LibraSpec reduces draft-model invocations and the associated I/O overhead.

% 说明长度一缩短就停止迭代的原因
LibraSpec initializes the iterative refinement from the default block size of each draft model, using it as the initial speculative length d. Inspired by the design philosophy of damping techniques in iterative optimization, LibraSpec performs at most one rollback once the estimated speculative length falls below the current length. Since repeated updates may introduce oscillatory behavior due to the interaction between speculative length adjustment and probability estimation, LibraSpec terminates the adjustment after a single correction. In practice, this design improves stability while avoiding additional computation that often provides limited benefit.

\begin{table*}[tbp]
%	\small
	\centering
	\caption{Comparison with Dynamic Speculative Length Methods on Qwen2.5-Instruct. FailFast, G4-style, and Ours are all built upon FastdLLM. (FdLLM: FastdLLM; $\eta$: speedup; $\tau$: avg.\ acceptance length.)}
	\label{tab:main_results_qwen25_combined}
	\setlength{\tabcolsep}{3pt}
	\renewcommand{\arraystretch}{1.08}
	\resizebox{\textwidth}{!}{%
		\begin{tabular}{cl*{20}{c}}
			\toprule
			\multirow{3}{*}{\rotatebox{90}{Model}} &
			\multirow{3}{*}{Method} &
			\multicolumn{10}{c}{Temperature = 0} &
			\multicolumn{10}{c}{Temperature = 1} \\
			\cmidrule(lr){3-12}\cmidrule(lr){13-22}
			& &
			\multicolumn{2}{c}{MATH-500} &
			\multicolumn{2}{c}{GSM8K} &
			\multicolumn{2}{c}{HumanEval} &
			\multicolumn{2}{c}{MT-Bench} &
			\multicolumn{2}{c}{Avg} &
			\multicolumn{2}{c}{MATH-500} &
			\multicolumn{2}{c}{GSM8K} &
			\multicolumn{2}{c}{HumanEval} &
			\multicolumn{2}{c}{MT-Bench} &
			\multicolumn{2}{c}{Avg} \\
			\cmidrule(lr){3-4}\cmidrule(lr){5-6}\cmidrule(lr){7-8}\cmidrule(lr){9-10}\cmidrule(lr){11-12}
			\cmidrule(lr){13-14}\cmidrule(lr){15-16}\cmidrule(lr){17-18}\cmidrule(lr){19-20}\cmidrule(lr){21-22}
			& &
			$\eta$ & $\tau$ & $\eta$ & $\tau$ & $\eta$ & $\tau$ & $\eta$ & $\tau$ & $\eta$ & $\tau$ &
			$\eta$ & $\tau$ & $\eta$ & $\tau$ & $\eta$ & $\tau$ & $\eta$ & $\tau$ & $\eta$ & $\tau$ \\
			\midrule

			\multirow{5}{*}{\rotatebox{90}{Q2.5-7B}}
			& EAGLE-3 & 2.23$\times$ & 3.71 & 2.12$\times$ & 3.52 & 2.25$\times$ & 3.21 & 1.79$\times$ & 2.91 & 2.10$\times$ & 3.34 & 2.12$\times$ & 3.58 & 2.03$\times$ & 3.42 & 2.16$\times$ & 3.06 & 1.66$\times$ & 2.77 & 1.99$\times$ & 3.21 \\
			& FdLLM & 1.92$\times$ & 3.29 & 1.53$\times$ & 2.25 & 1.61$\times$ & 2.43 & 1.40$\times$ & 2.03 & 1.62$\times$ & 2.50 & 1.81$\times$ & 3.16 & 1.48$\times$ & 2.16 & 1.47$\times$ & 2.27 & 1.29$\times$ & 1.87 & 1.51$\times$ & 2.37 \\
			& FdLLM+FailFast & 2.98$\times$ & 4.48 & 2.25$\times$ & 3.79 & 2.52$\times$ & 4.83 & 1.93$\times$ & 3.16 & 2.42$\times$ & 4.07 & 2.79$\times$ & 3.87 & 2.07$\times$ & 3.65 & 2.33$\times$ & 4.66 & 1.79$\times$ & 3.02 & 2.25$\times$ & 3.80 \\
			& FdLLM+G4-style & 2.29$\times$ & 3.80 & 1.76$\times$ & 2.64 & 2.02$\times$ & 3.80 & 1.40$\times$ & 2.05 & 1.87$\times$ & 3.07 & 2.10$\times$ & 3.51 & 1.78$\times$ & 2.86 & 1.75$\times$ & 3.48 & 1.26$\times$ & 1.82 & 1.72$\times$ & 2.92 \\
			& FdLLM+Ours & \textbf{3.54$\times$} & \textbf{5.11} & \textbf{3.09$\times$} & \textbf{4.23} & \textbf{3.12$\times$} & \textbf{5.29} & \textbf{2.39$\times$} & \textbf{3.43} & \textbf{3.04$\times$} & \textbf{4.52} & \textbf{3.31$\times$} & \textbf{4.94} & \textbf{2.88$\times$} & \textbf{4.05} & \textbf{3.00$\times$} & \textbf{5.09} & \textbf{2.14$\times$} & \textbf{3.29} & \textbf{2.83$\times$} & \textbf{4.34} \\
			\midrule

			\multirow{5}{*}{\rotatebox{90}{Q2.5-14B}}
			& EAGLE-3 & 2.36$\times$ & 3.80 & 2.19$\times$ & 3.59 & 2.61$\times$ & 3.83 & 1.81$\times$ & 2.99 & 2.24$\times$ & 3.55 & 2.25$\times$ & 3.69 & 2.06$\times$ & 3.48 & 2.48$\times$ & 3.73 & 1.66$\times$ & 2.82 & 2.11$\times$ & 3.43 \\
			& FdLLM & 2.47$\times$ & 4.19 & 2.21$\times$ & 3.68 & 2.23$\times$ & 2.94 & 1.93$\times$ & 2.76 & 2.21$\times$ & 3.39 & 2.40$\times$ & 4.08 & 2.09$\times$ & 3.50 & 2.11$\times$ & 2.83 & 1.82$\times$ & 2.61 & 2.11$\times$ & 3.26 \\
			& FdLLM+FailFast & 3.73$\times$ & 6.38 & 2.87$\times$ & 4.37 & 3.41$\times$ & 5.12 & 2.25$\times$ & 3.37 & 3.07$\times$ & 4.81 & 3.56$\times$ & 6.11 & 2.65$\times$ & 4.22 & 3.23$\times$ & 4.92 & 1.99$\times$ & 3.17 & 2.86$\times$ & 4.61 \\
			& FdLLM+G4-style & 2.59$\times$ & 4.37 & 1.82$\times$ & 3.27 & 2.76$\times$ & 3.89 & 2.04$\times$ & 2.95 & 2.30$\times$ & 3.62 & 2.66$\times$ & 4.24 & 2.18$\times$ & 3.73 & 2.41$\times$ & 3.48 & 1.87$\times$ & 2.70 & 2.28$\times$ & 3.54 \\
			& FdLLM+Ours & \textbf{4.71$\times$} & \textbf{7.31} & \textbf{3.68$\times$} & \textbf{5.23} & \textbf{4.08$\times$} & \textbf{5.97} & \textbf{2.64$\times$} & \textbf{3.72} & \textbf{3.78$\times$} & \textbf{5.56} & \textbf{4.51$\times$} & \textbf{6.96} & \textbf{3.49$\times$} & \textbf{4.98} & \textbf{3.92$\times$} & \textbf{5.84} & \textbf{2.42$\times$} & \textbf{3.61} & \textbf{3.59$\times$} & \textbf{5.35} \\
			\midrule

			\multirow{5}{*}{\rotatebox{90}{Q2.5-32B}}
			& EAGLE-3 & 2.50$\times$ & 3.95 & 2.27$\times$ & 3.69 & 2.68$\times$ & 3.86 & 1.90$\times$ & 3.14 & 2.34$\times$ & 3.66 & 2.39$\times$ & 3.76 & 2.14$\times$ & 3.50 & 2.56$\times$ & 3.72 & 1.74$\times$ & 2.95 & 2.21$\times$ & 3.48 \\
			& FdLLM & 3.57$\times$ & 4.75 & 3.13$\times$ & 4.38 & 3.16$\times$ & 3.97 & 2.18$\times$ & 2.98 & 3.01$\times$ & 4.02 & 3.31$\times$ & 4.59 & 2.84$\times$ & 4.03 & 2.83$\times$ & 3.82 & 1.99$\times$ & 2.84 & 2.74$\times$ & 3.82 \\
			& FdLLM+FailFast & 4.90$\times$ & 6.45 & 3.71$\times$ & 4.59 & 4.06$\times$ & 5.15 & 2.41$\times$ & 3.49 & 3.77$\times$ & 4.87 & 4.67$\times$ & 6.08 & 3.52$\times$ & 4.16 & 3.82$\times$ & 4.97 & 2.20$\times$ & 3.32 & 3.55$\times$ & 4.63 \\
			& FdLLM+G4-style & 3.89$\times$ & 4.98 & 3.20$\times$ & 4.39 & 3.31$\times$ & 4.18 & 2.23$\times$ & 3.03 & 3.16$\times$ & 4.15 & 3.48$\times$ & 4.86 & 3.18$\times$ & 4.10 & 2.88$\times$ & 3.90 & 1.99$\times$ & 2.82 & 2.88$\times$ & 3.92 \\
			& FdLLM+Ours & \textbf{5.45$\times$} & \textbf{7.32} & \textbf{4.11$\times$} & \textbf{4.78} & \textbf{4.70$\times$} & \textbf{5.94} & \textbf{2.93$\times$} & \textbf{4.25} & \textbf{4.30$\times$} & \textbf{5.57} & \textbf{5.11$\times$} & \textbf{6.99} & \textbf{3.87$\times$} & \textbf{4.32} & \textbf{4.41$\times$} & \textbf{5.63} & \textbf{2.58$\times$} & \textbf{3.98} & \textbf{3.99$\times$} & \textbf{5.23} \\
			\bottomrule
		\end{tabular}%
	}
\end{table*}

\section{Experiments}

\subsection{Experimental Setup}

% 我们基于通义千问2.5（7B指令版、14B指令版、32B指令版）与通义千问3（4B、8B、Coder-30B-A3B指令版）预训练模型开展实验。我们在三类基准测试数据集上对所提方法开展评估：数学类（MATH-500、GSM8K）、代码类（HumanEval）以及对话类（MT-Bench）。针对每一项基准测试，我们通过相较于自回归基线模型的端到端加速比与平均接受长度（$\tau$）来评估解码效率。

\textbf{Models and Evaluations.} We conduct experiments on Qwen2.5-\{7, 14, 32\}B-Instruct \citep{qwen2.5-qwen} and Qwen3-\{4B, 8B, Coder-30B-A3B-Instruct\} \citep{qwen3-qwen} pre-trained models. We evaluate our method on benchmarks spanning three categories: Math, including MATH-500 \citep{math500-hunter} and GSM8K \citep{gsm8k-cobbe}; Code, including HumanEval \citep{humaneval-chen}; and Chat, including MT-Bench \citep{mtbench-zheng}. For each benchmark, we assess decoding efficiency using end-to-end speedup ($\eta$) over the autoregressive baseline and average acceptance length ($\tau$).

% 我们采用原生自回归解码作为基线，将其作为加速比（1.00倍）的参照标准。
% 我们选取具有代表性的推测解码方法开展对比实验，包括自回归方法\textsc{EAGLE-3}、基于扩散的方法\textsc{DFlash}与\textsc{DDTree}、扩散草稿模型\textsc{Fast-dLLM-v2-1.5B}，以及将\textsc{Fast-dLLM-v2-1.5B}用作草稿生成器的\textsc{FailFast}。由于\textsc{LibraSpec}是一种可即插即用的加速技术，我们进一步将其与现有推测解码框架结合开展实验评估，以此验证该技术的兼容性以及可带来的增益效果。

\begin{table*}[tbp]
	\centering
	\caption{Performance of LibraSpec Integrated with Diffusion-Based Speculative Decoding Methods (DFlash and DDTree).}
	\label{tab:main_results_qwen3_combined}
	\setlength{\tabcolsep}{3pt}
	\renewcommand{\arraystretch}{1.08}
	\resizebox{\textwidth}{!}{%
		\begin{tabular}{cl*{20}{c}}
			\toprule
			\multirow{3}{*}{Model} &
			\multirow{3}{*}{Method} &
			\multicolumn{10}{c}{Temperature = 0} &
			\multicolumn{10}{c}{Temperature = 1} \\
			\cmidrule(lr){3-12}\cmidrule(lr){13-22}
			& &
			\multicolumn{2}{c}{MATH-500} &
			\multicolumn{2}{c}{GSM8K} &
			\multicolumn{2}{c}{HumanEval} &
			\multicolumn{2}{c}{MT-Bench} &
			\multicolumn{2}{c}{Avg} &
			\multicolumn{2}{c}{MATH-500} &
			\multicolumn{2}{c}{GSM8K} &
			\multicolumn{2}{c}{HumanEval} &
			\multicolumn{2}{c}{MT-Bench} &
			\multicolumn{2}{c}{Avg} \\
			\cmidrule(lr){3-4}\cmidrule(lr){5-6}\cmidrule(lr){7-8}\cmidrule(lr){9-10}\cmidrule(lr){11-12}
			\cmidrule(lr){13-14}\cmidrule(lr){15-16}\cmidrule(lr){17-18}\cmidrule(lr){19-20}\cmidrule(lr){21-22}
			& &
			$\eta$ & $\tau$ & $\eta$ & $\tau$ & $\eta$ & $\tau$ & $\eta$ & $\tau$ & $\eta$ & $\tau$ &
			$\eta$ & $\tau$ & $\eta$ & $\tau$ & $\eta$ & $\tau$ & $\eta$ & $\tau$ & $\eta$ & $\tau$ \\
			\midrule

			\multirow{4}{*}{\rotatebox{90}{Q3-4B}}
			& DFlash & 5.54$\times$ & 7.73 & 5.10$\times$ & 6.50 & 4.81$\times$ & 6.52 & 2.64$\times$ & 4.33 & 4.52$\times$ & 6.27 & 4.50$\times$ & 6.61 & 4.31$\times$ & 5.97 & 4.36$\times$ & 5.99 & 2.54$\times$ & 4.08 & 3.93$\times$ & 5.66 \\
			& DFlash + Ours & \textbf{5.98$\times$} & \textbf{8.04} & \textbf{5.54$\times$} & \textbf{7.74} & \textbf{5.58$\times$} & \textbf{7.83} & \textbf{2.94$\times$} & \textbf{4.93} & \textbf{5.01$\times$} & \textbf{7.14} & \textbf{4.95$\times$} & \textbf{7.03} & \textbf{4.72$\times$} & \textbf{6.61} & \textbf{5.11$\times$} & \textbf{7.12} & \textbf{2.82$\times$} & \textbf{4.58} & \textbf{4.40$\times$} & \textbf{6.34} \\
			& DDTree & 7.26$\times$ & 10.21 & 6.37$\times$ & 9.31 & 6.81$\times$ & 9.22 & 4.16$\times$ & 6.59 & 6.15$\times$ & 8.83 & 6.52$\times$ & 9.56 & 5.88$\times$ & 8.84 & 5.94$\times$ & 9.15 & 3.79$\times$ & 6.24 & 5.53$\times$ & 8.45 \\
			& DDTree + Ours & \textbf{8.03$\times$} & \textbf{10.97} & \textbf{7.18$\times$} & \textbf{9.98} & \textbf{7.36$\times$} & \textbf{10.27} & \textbf{4.77$\times$} & \textbf{6.87} & \textbf{6.84$\times$} & \textbf{9.52} & \textbf{6.96$\times$} & \textbf{9.87} & \textbf{6.49$\times$} & \textbf{9.50} & \textbf{6.53$\times$} & \textbf{9.92} & \textbf{4.27$\times$} & \textbf{6.91} & \textbf{6.06$\times$} & \textbf{9.05} \\
			\midrule

			\multirow{4}{*}{\rotatebox{90}{Q3-8B}}
			& DFlash & 5.56$\times$ & 7.80 & 4.84$\times$ & 6.55 & 4.73$\times$ & 6.58 & 2.58$\times$ & 4.24 & 4.43$\times$ & 6.29 & 4.52$\times$ & 6.46 & 3.93$\times$ & 5.91 & 4.17$\times$ & 5.49 & 2.28$\times$ & 3.78 & 3.73$\times$ & 5.41 \\
			& DFlash+Ours & \textbf{6.17$\times$} & \textbf{8.27} & \textbf{5.66$\times$} & \textbf{7.91} & \textbf{5.72$\times$} & \textbf{7.99} & \textbf{2.91$\times$} & \textbf{4.61} & \textbf{5.12$\times$} & \textbf{7.20} & \textbf{5.14$\times$} & \textbf{7.38} & \textbf{4.47$\times$} & \textbf{6.74} & \textbf{5.08$\times$} & \textbf{6.44} & \textbf{2.65$\times$} & \textbf{4.29} & \textbf{4.34$\times$} & \textbf{6.21} \\
			& DDTree & 7.52$\times$ & 10.53 & 6.45$\times$ & 9.45 & 6.80$\times$ & 9.68 & 4.15$\times$ & 6.59 & 6.23$\times$ & 9.06 & 6.50$\times$ & 9.52 & 5.96$\times$ & 8.90 & 5.96$\times$ & 8.53 & 3.51$\times$ & 5.97 & 5.48$\times$ & 8.23 \\
			& DDTree+Ours & \textbf{8.03$\times$} & \textbf{11.08} & \textbf{7.24$\times$} & \textbf{10.08} & \textbf{7.28$\times$} & \textbf{10.15} & \textbf{4.69$\times$} & \textbf{6.78} & \textbf{6.81$\times$} & \textbf{9.52} & \textbf{7.07$\times$} & \textbf{10.10} & \textbf{6.51$\times$} & \textbf{9.43} & \textbf{6.50$\times$} & \textbf{9.18} & \textbf{4.09$\times$} & \textbf{6.51} & \textbf{6.04$\times$} & \textbf{8.81} \\
			\midrule

			\multirow{4}{*}{\rotatebox{90}{\shortstack{Q3-Coder\\30B-A3B}}}
			& DFlash & 4.29$\times$ & 5.56 & 4.01$\times$ & 5.18 & 6.09$\times$ & 7.99 & 2.04$\times$ & 3.52 & 4.11$\times$ & 5.56 & 4.02$\times$ & 5.31 & 3.89$\times$ & 5.06 & 5.69$\times$ & 7.58 & 1.88$\times$ & 3.39 & 3.87$\times$ & 5.34 \\
			& DFlash+Ours & \textbf{4.68$\times$} & \textbf{6.22} & \textbf{4.45$\times$} & \textbf{5.77} & \textbf{6.51$\times$} & \textbf{8.56} & \textbf{2.40$\times$} & \textbf{4.27} & \textbf{4.51$\times$} & \textbf{6.21} & \textbf{4.51$\times$} & \textbf{6.03} & \textbf{4.26$\times$} & \textbf{5.50} & \textbf{6.17$\times$} & \textbf{8.24} & \textbf{2.10$\times$} & \textbf{3.93} & \textbf{4.26$\times$} & \textbf{5.93} \\
			& DDTree & 6.10$\times$ & 8.04 & 5.83$\times$ & 7.48 & 8.12$\times$ & 10.33 & 3.07$\times$ & 5.32 & 5.78$\times$ & 7.79 & 5.82$\times$ & 7.68 & 5.63$\times$ & 7.29 & 7.71$\times$ & 10.01 & 2.96$\times$ & 5.12 & 5.53$\times$ & 7.53 \\
			& DDTree+Ours & \textbf{6.58$\times$} & \textbf{8.58} & \textbf{6.25$\times$} & \textbf{8.09} & \textbf{8.49$\times$} & \textbf{10.67} & \textbf{3.38$\times$} & \textbf{5.81} & \textbf{6.18$\times$} & \textbf{8.29} & \textbf{6.24$\times$} & \textbf{8.17} & \textbf{6.15$\times$} & \textbf{7.83} & \textbf{8.21$\times$} & \textbf{10.47} & \textbf{3.33$\times$} & \textbf{5.76} & \textbf{5.98$\times$} & \textbf{8.06} \\
			\bottomrule
		\end{tabular}%
	}
\end{table*}

\textbf{Baselines.} We use vanilla autoregressive decoding as the baseline, which serves as the benchmark for speedup ratios (1.00$\times$). We compare against representative speculative decoding approaches, including the autoregressive method EAGLE-3 \citep{eagle3-li} and diffusion-based approaches Fast-dLLM-v2-1.5B \citep{fastdllmv2-wu} (hereinafter referred to as FastdLLM), FailFast \citep{failfast-pan}, DFlash \citep{dflash-chen}, and DDTree \citep{ddtree-ringel}. We additionally adapt the dynamic speculative length strategy used in Gemma 4 MTP \cite{gemma4-mtp} to the FastdLLM for comparison with representative dynamic speculative length methods (hereinafter referred to as G4-style). Since our method is a plug-and-play acceleration technique, we further evaluate its integration with existing speculative decoding methods to assess its compatibility and complementary benefits.

% 除非另有说明，$\alpha$ 选取 $\{1.8, 1.9, 2.0, 2.1, 2.2\}$ 中加速效果最优的。所有实验均基于英伟达A100图形处理器开展。

\textbf{Implementation Details.} Unless otherwise stated, all experiments are conducted on NVIDIA A100 GPUs. Details of the configurations for each benchmark and draft models are provided in Appendix~\ref{sc:config-details}.

\subsection{Comparison with Dynamic Speculative Length Methods}

In this section, we evaluate LibraSpec against existing dynamic speculative length methods. We adopt FastdLLM, a widely used diffusion-based drafter, and follow its officially released Qwen2.5 configuration as our experimental setup. The dynamic speculative length baselines include FailFast, a dynamic speculation strategy specifically designed for FastdLLM, and G4-style, the speculative length strategy proposed in Gemma 4, which is adapted to FastdLLM for a fair comparison. We further include EAGLE-3 as a representative autoregressive speculative decoding baseline.

As shown in Table~\ref{tab:main_results_qwen25_combined}, LibraSpec consistently outperforms all existing dynamic speculative-length strategies across all Qwen2.5 models, benchmark datasets, and decoding settings. Under greedy decoding ($\mathrm{temperature}=0$), integrating LibraSpec with FastdLLM improves the average end-to-end speedup by 1.42$\times$, 1.57$\times$, and 1.29$\times$ on Qwen2.5-7B, 14B, and 32B, respectively, while consistently surpassing FailFast and G4-style. The performance gains remain robust under non-greedy sampling ($\mathrm{temperature}=1$), where LibraSpec further improves the average speedup of FastdLLM by 1.32$\times$, 1.48$\times$, and 1.25$\times$, respectively. These results demonstrate that, compared with existing methods, LibraSpec enables substantially more effective speculative length adjustment, and that its performance gains consistently generalize across both greedy decoding and stochastic sampling.

\begin{table*}[htbp]
	\centering
	\caption{Performance Comparison of Reasoning Models on Different Benchmarks.}
	\label{tab:main_results_thinking}
	\setlength{\tabcolsep}{3pt}
	\renewcommand{\arraystretch}{1.08}
	\resizebox{\textwidth}{!}{%
		\begin{tabular}{cl*{20}{c}}
			\toprule
			\multirow{3}{*}{Model} &
			\multirow{3}{*}{Method} &
			\multicolumn{10}{c}{Temperature = 0} &
			\multicolumn{10}{c}{Temperature = 1} \\
			\cmidrule(lr){3-12}\cmidrule(lr){13-22}
			& &
			\multicolumn{2}{c}{MATH-500} &
			\multicolumn{2}{c}{GSM8K} &
			\multicolumn{2}{c}{HumanEval} &
			\multicolumn{2}{c}{MT-Bench} &
			\multicolumn{2}{c}{Avg} &
			\multicolumn{2}{c}{MATH-500} &
			\multicolumn{2}{c}{GSM8K} &
			\multicolumn{2}{c}{HumanEval} &
			\multicolumn{2}{c}{MT-Bench} &
			\multicolumn{2}{c}{Avg} \\
			\cmidrule(lr){3-4}\cmidrule(lr){5-6}\cmidrule(lr){7-8}\cmidrule(lr){9-10}\cmidrule(lr){11-12}
			\cmidrule(lr){13-14}\cmidrule(lr){15-16}\cmidrule(lr){17-18}\cmidrule(lr){19-20}\cmidrule(lr){21-22}
			& &
			$\eta$ & $\tau$ & $\eta$ & $\tau$ & $\eta$ & $\tau$ & $\eta$ & $\tau$ & $\eta$ & $\tau$ &
			$\eta$ & $\tau$ & $\eta$ & $\tau$ & $\eta$ & $\tau$ & $\eta$ & $\tau$ & $\eta$ & $\tau$ \\
			\midrule
			
			\multirow{2}{*}{\rotatebox{90}{Q3-4B}}
			& DFlash & 3.92$\times$ & 5.72 & 3.63$\times$ & 5.01 & 3.38$\times$ & 4.71 & 2.18$\times$ & 3.11 & 3.28$\times$ & 4.64 & 3.39$\times$ & 4.84 & 3.34$\times$ & 4.68 & 3.09$\times$ & 4.28 & 2.09$\times$ & 3.01 & 2.98$\times$ & 4.20 \\
			& DFlash+Ours & \textbf{4.31$\times$} & \textbf{6.09} & \textbf{3.96$\times$} & \textbf{5.57} & \textbf{3.66$\times$} & \textbf{5.03} & \textbf{2.39$\times$} & \textbf{3.49} & \textbf{3.58$\times$} & \textbf{5.05} & \textbf{3.71$\times$} & \textbf{5.29} & \textbf{3.73$\times$} & \textbf{5.25} & \textbf{3.34$\times$} & \textbf{4.74} & \textbf{2.31$\times$} & \textbf{3.38} & \textbf{3.27$\times$} & \textbf{4.67} \\
			\midrule
			
			\multirow{2}{*}{\rotatebox{90}{Q3-8B}}
			& DFlash & 4.01$\times$ & 5.82 & 3.62$\times$ & 5.15 & 3.32$\times$ & 4.70 & 2.31$\times$ & 3.37 & 3.32$\times$ & 4.76 & 3.53$\times$ & 5.06 & 3.31$\times$ & 4.69 & 3.03$\times$ & 4.29 & 2.15$\times$ & 3.13 & 3.01$\times$ & 4.29 \\
			& DFlash+Ours & \textbf{4.39$\times$} & \textbf{6.18} & \textbf{3.91$\times$} & \textbf{5.60} & \textbf{3.61$\times$} & \textbf{5.07} & \textbf{2.51$\times$} & \textbf{3.62} & \textbf{3.61$\times$} & \textbf{5.12} & \textbf{3.87$\times$} & \textbf{5.42} & \textbf{3.60$\times$} & \textbf{5.09} & \textbf{3.30$\times$} & \textbf{4.67} & \textbf{2.33$\times$} & \textbf{3.42} & \textbf{3.28$\times$} & \textbf{4.65} \\
			\bottomrule
		\end{tabular}%
	}
\end{table*}

\begin{table*}[htbp]
	\centering
	\caption{Sensitivity of LibraSpec to the Hyperparameter $\alpha$,
		Compared with Static Speculative Length Baselines.}
	\label{tab:ablation-alpha}
	\small
	\setlength{\tabcolsep}{3pt}
	\resizebox{\textwidth}{!}{%
		\begin{tabular}{cccccccccccc@{\hspace{8pt}}cccccccccccc}
			\toprule
			\multicolumn{12}{c}{Target: Q2.5-7B \quad Draft: FastdLLM} &
			\multicolumn{12}{c}{Target: Q3-8B \quad Draft: DFlash} \\
			\cmidrule(lr){1-12}\cmidrule(lr){13-24}
			
			\multirow{2}{*}{w/ Ours} &
			\multirow{2}{*}{$\alpha$} &
			\multicolumn{2}{c}{MATH-500} &
			\multicolumn{2}{c}{GSM8K} &
			\multicolumn{2}{c}{HumanEval} &
			\multicolumn{2}{c}{MT-Bench} &
			\multicolumn{2}{c}{Avg} &
			\multirow{2}{*}{w/ Ours} &
			\multirow{2}{*}{$\alpha$} &
			\multicolumn{2}{c}{MATH-500} &
			\multicolumn{2}{c}{GSM8K} &
			\multicolumn{2}{c}{HumanEval} &
			\multicolumn{2}{c}{MT-Bench} &
			\multicolumn{2}{c}{Avg} \\
			
			& & $\eta$ & $\tau$ & $\eta$ & $\tau$ &
			$\eta$ & $\tau$ & $\eta$ & $\tau$ & $\eta$ & $\tau$ &
			& & $\eta$ & $\tau$ & $\eta$ & $\tau$ &
			$\eta$ & $\tau$ & $\eta$ & $\tau$ & $\eta$ & $\tau$ \\
			\midrule
			
			$\times$ & / &
			1.92$\times$ & 3.29 &
			1.53$\times$ & 2.25 &
			1.61$\times$ & 2.43 &
			1.40$\times$ & 2.03 &
			1.62$\times$ & 2.50 &
			$\times$ & / &
			5.56$\times$ & 7.80 &
			4.84$\times$ & 6.55 &
			4.73$\times$ & 6.58 &
			2.58$\times$ & 4.24 &
			4.43$\times$ & 6.29 \\
			
			\cmidrule(lr){1-12}\cmidrule(lr){13-24}
			
			\multirow{5}{*}{\checkmark} & 1.8 &
			3.48$\times$ & 5.01 &
			2.94$\times$ & 4.05 &
			3.01$\times$ & 5.18 &
			2.36$\times$ & 3.40 &
			2.95$\times$ & 4.41 &
			\multirow{5}{*}{\checkmark} & 1.8 &
			5.88$\times$ & 7.98 &
			5.36$\times$ & 7.48 &
			5.47$\times$ & 7.56 &
			2.72$\times$ & 4.33 &
			4.86$\times$ & 6.84 \\
			
			& 1.9 &
			3.51$\times$ & 5.05 &
			3.03$\times$ & 4.15 &
			3.11$\times$ & 5.27 &
			\textbf{2.39$\times$} & \textbf{3.43} &
			3.01$\times$ & 4.48 &
			& 1.9 &
			5.99$\times$ & 8.09 &
			5.42$\times$ & 7.63 &
			5.58$\times$ & 7.76 &
			2.80$\times$ & 4.47 &
			4.95$\times$ & 6.99 \\
			
			& 2.0 &
			\textbf{3.54$\times$} & \textbf{5.11} &
			\textbf{3.09$\times$} & \textbf{4.23} &
			\textbf{3.12$\times$} & \textbf{5.29} &
			\textbf{2.39$\times$} & \textbf{3.43} &
			\textbf{3.04$\times$} & \textbf{4.52} &
			& 2.0 &
			6.05$\times$ & 8.16 &
			5.50$\times$ & 7.69 &
			5.63$\times$ & 7.83 &
			\textbf{2.91$\times$} & \textbf{4.61} &
			5.02$\times$ & 7.07 \\
			
			& 2.1 &
			3.49$\times$ & 5.01 &
			3.02$\times$ & 4.14 &
			3.05$\times$ & 5.21 &
			2.29$\times$ & 3.38 &
			2.96$\times$ & 4.44 &
			& 2.1 &
			6.15$\times$ & 8.24 &
			5.58$\times$ & 7.84 &
			5.71$\times$ & 7.96 &
			\textbf{2.91$\times$} & \textbf{4.61} &
			5.09$\times$ & 7.16 \\
			
			& 2.2 &
			3.43$\times$ & 4.89 &
			2.92$\times$ & 4.03 &
			2.96$\times$ & 5.13 &
			2.17$\times$ & 3.32 &
			2.87$\times$ & 4.34 &
			& 2.2 &
			\textbf{6.17$\times$} & \textbf{8.27} &
			\textbf{5.66$\times$} & \textbf{7.91} &
			\textbf{5.72$\times$} & \textbf{7.99} &
			2.89$\times$ & 4.57 &
			\textbf{5.11$\times$} & \textbf{7.19} \\
			
			\bottomrule
		\end{tabular}%
	}
\end{table*}

\subsection{Integration with State-of-the-Art Diffusion-Based Drafters}

We further integrate LibraSpec into state-of-the-art diffusion-based speculative decoding methods, including DFlash and DDTree, and evaluate them on the Qwen3 models with thinking mode disabled. Following their official configurations, we evaluate the decoding speedup based on their publicly released Qwen3 draft models.

As shown in Table~\ref{tab:main_results_qwen3_combined}, LibraSpec consistently improves the performance of both DFlash and DDTree across all Qwen3 models, benchmarks, and decoding settings. Compared with the original DFlash, LibraSpec provides an additional average decoding speedup of $0.39\sim0.69\times$ under greedy decoding and $0.39\sim0.67\times$ under non-greedy sampling. Similarly, LibraSpec brings an additional average decoding speedup of $0.39\sim0.69\times$ and $0.45\sim0.56\times$ over DDTree under greedy decoding and non-greedy sampling, respectively. These results demonstrate that LibraSpec is complementary to existing diffusion-based speculative decoding methods and consistently enhances their decoding efficiency across different model architectures and sampling strategies.

\subsection{Robustness under Thinking-Mode Decoding}

% In this section, we evaluate LibraSpec on the Qwen3 models with thinking mode enabled. As shown in Table~\ref{tab:main_results_thinking}, across both greedy decoding ($\text{temperature}=0$) and non-greedy sampling ($\text{temperature}=1$), LibraSpec consistently achieves an additional $8\% \sim 11\%$ end-to-end speedup over the original DFlash.

In this section, we evaluate LibraSpec on the Qwen3 models with thinking mode enabled. As shown in Table~\ref{tab:main_results_thinking}, LibraSpec consistently improves the performance of DFlash across all reasoning models, benchmark datasets, and decoding settings. Compared with the original DFlash, LibraSpec provides an additional average decoding speedup of approximately $0.29\sim0.30\times$ under greedy decoding ($\text{temperature}=0$) and $0.27\sim0.29\times$ under non-greedy sampling ($\text{temperature}=1$), corresponding to a relative improvement of approximately $9\%\sim11\%$ over the original methods.

Notably, the improvements remain evident even under reasoning-intensive scenarios, where the generated outputs typically exhibit higher entropy and greater uncertainty, making speculative decoding more challenging. Nevertheless, LibraSpec consistently improves decoding speed across all evaluated settings, demonstrating that the proposed dynamic speculative-length adjustment generalizes well to reasoning workloads.

\subsection{Hyperparameter ($\alpha$) Sensitivity Analysis}

We study the sensitivity of LibraSpec to the trade-off hyperparameter $\alpha$ under greedy decoding ($\mathrm{temperature}=0$), which balances the expected acceptance gain against the verification overhead. Specifically, a larger $\alpha$ encourages more aggressive draft expansion by placing greater weight on the acceptance gain, while a smaller $\alpha$ favors a more conservative strategy by assigning a higher penalty to failed expansions.

As shown in Table~\ref{tab:ablation-alpha}, LibraSpec is robust across a broad range of $\alpha$ values, requiring little hyperparameter tuning in practice. Moreover, compared with the corresponding baselines without LibraSpec, all evaluated $\alpha$ values consistently improve both decoding speedup and accepted speculative length, demonstrating that LibraSpec remains effective across a wide range of hyperparameter settings. For Qwen2.5-7B with FastdLLM, the best average speedup is achieved at $\alpha=2.0$, while larger values slightly reduce both speedup and $\tau$. 

In contrast, for Qwen3-8B with DFlash, performance consistently improves as $\alpha$ increases, with the best results obtained at $\alpha=2.2$. This suggests that stronger drafters can benefit from more aggressive speculative expansion because the cost of invalidated draft tokens is lower. Based on these observations, we use $\alpha=2.0$ for FastdLLM, and $\alpha=2.2$ for both DFlash and DDTree in all experiments.

\begin{figure*}[htbp]
	\centering
	
	\begin{subfigure}{\linewidth}
		\centering
		\includegraphics[width=0.99\linewidth]{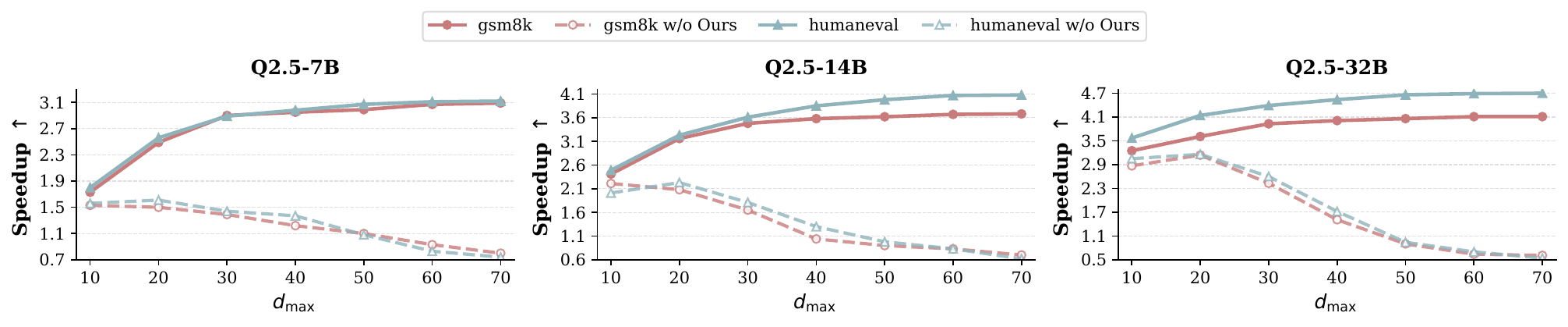}
		\caption{FastdLLM}
		\label{fig:sub1_d_max}
	\end{subfigure}
	
	%	\vspace{0.3em}
	
	\begin{subfigure}{\linewidth}
		\centering
		\includegraphics[width=0.99\linewidth]{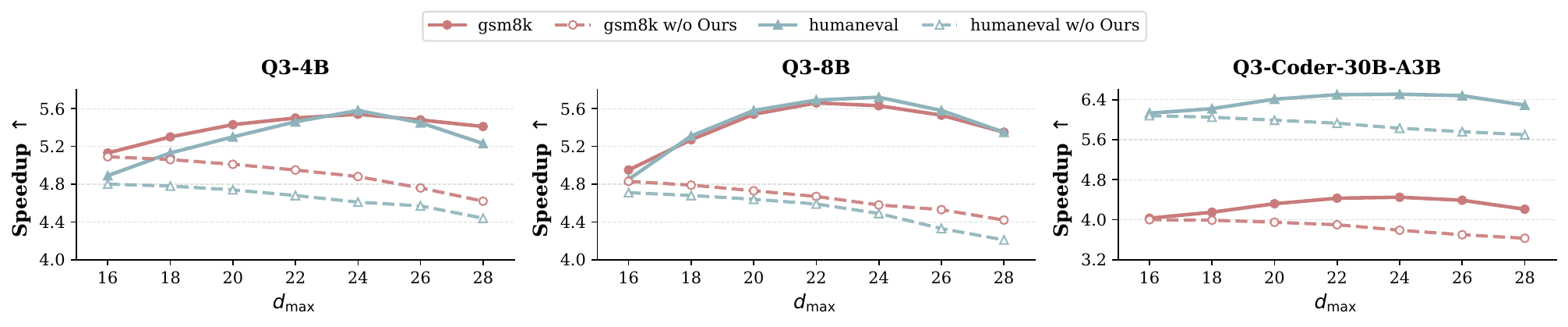}
		\caption{DFlash}
		\label{fig:sub2_d_max}
	\end{subfigure}
	
	\caption{Effect of Maximum Speculative Length on the Speedup of LibraSpec, Compared with Static Speculative Length.}
	\label{fig:d_max_speedup}
\end{figure*}

\subsection{Maximum speculative Length ($d_{\text{max}}$) Sensitivity Analysis}
\label{sc:maxd}

In this section, we investigate how the maximum speculative length $d_{\mathrm{max}}$ affects decoding efficiency for different draft models. As shown in Figure~\ref{fig:d_max_speedup}, regardless of the choice of $d_{\mathrm{max}}$, LibraSpec consistently outperforms the original static length decoding strategy (denoted as w/o Ours in the figure) across all evaluated draft models and datasets. 
The performance gap becomes even more pronounced as the maximum speculative length increases, highlighting the advantage of dynamic speculative length adjustment over static length decoding under more aggressive speculation.
This demonstrates that LibraSpec is robust to the choice of $d_{\mathrm{max}}$, consistently delivering speedup improvements under different maximum speculative length configurations.

Moreover, as draft confidence calibration deteriorates over longer speculative horizons, different draft models exhibit different effective ranges of $d_{\mathrm{max}}$. For FastdLLM, speedup increases before plateauing at approximately $d_{\mathrm{max}}=50\sim70$, suggesting reliable calibration over relatively long horizons. Based on this phenomenon, we therefore set $d_{\mathrm{max}}=60$ for all FastdLLM experiments.

In contrast, DFlash has a shorter effective range because its reliance on target model hidden states inherently limits the parallel decoding block size. Beyond approximately $22\sim24$ tokens, the calibration of the draft-model distribution $q$ against the target-model distribution $p$ gradually deteriorates, leading to larger prediction errors and lower acceptance rates; further extending the speculative length therefore yields little additional benefit, while Assumption~\ref{ass:specforge} becomes less valid. We therefore set $d_{\mathrm{max}}=24$ for all DFlash and DDTree experiments.

\begin{figure}[tbp]
    \centering

    % 第一行
    \begin{subfigure}{0.49\linewidth}
        \centering
        \includegraphics[width=\linewidth]{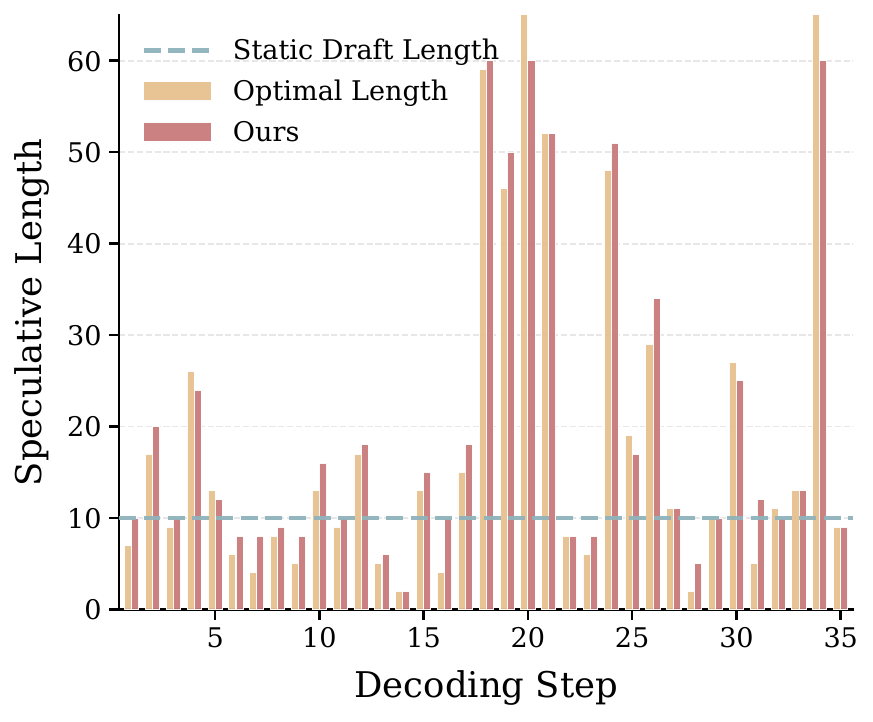}
        \caption{MATH-500}
        \label{fig:sub1}
    \end{subfigure}
    \hfill
    \begin{subfigure}{0.49\linewidth}
        \centering
        \includegraphics[width=\linewidth]{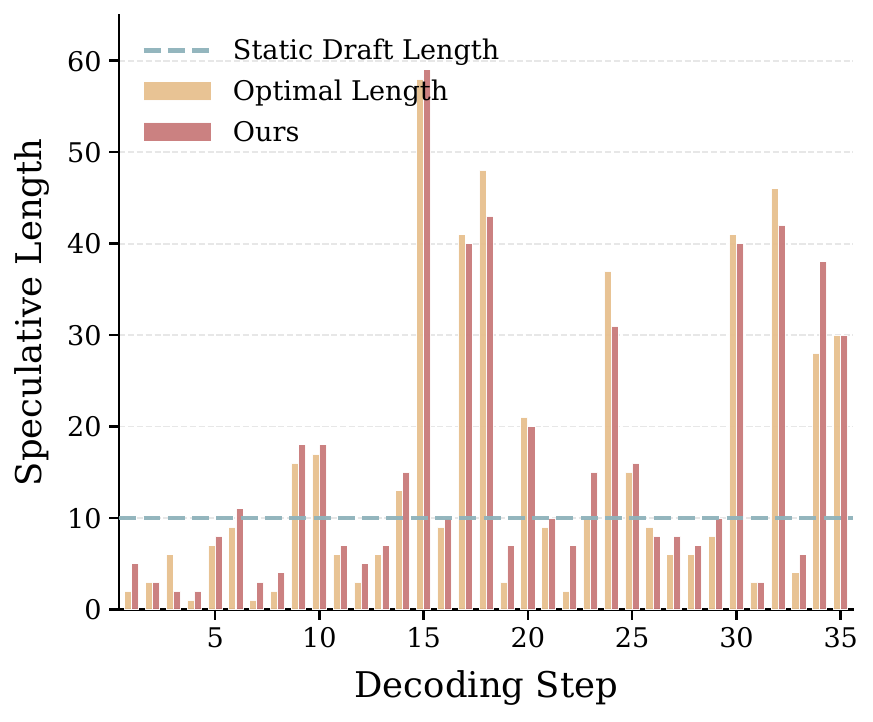}
        \caption{GSM8K}
        \label{fig:sub2}
    \end{subfigure}

    % \vspace{0.5em}

    % 第二行
    \begin{subfigure}{0.49\linewidth}
        \centering
        \includegraphics[width=\linewidth]{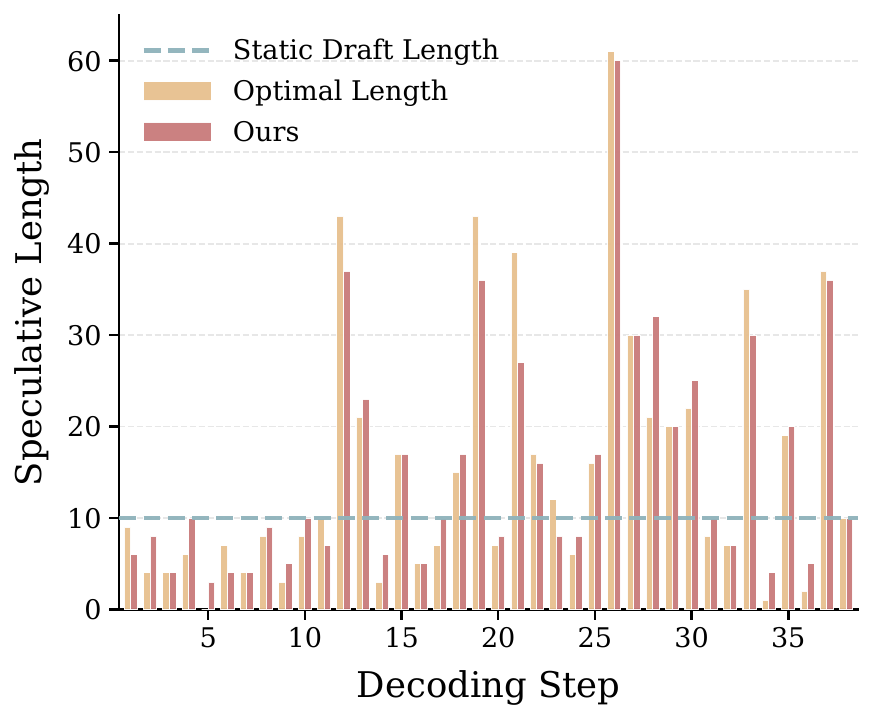}
        \caption{HumanEval}
        \label{fig:sub3}
    \end{subfigure}
    \hfill
    \begin{subfigure}{0.49\linewidth}
        \centering
        \includegraphics[width=\linewidth]{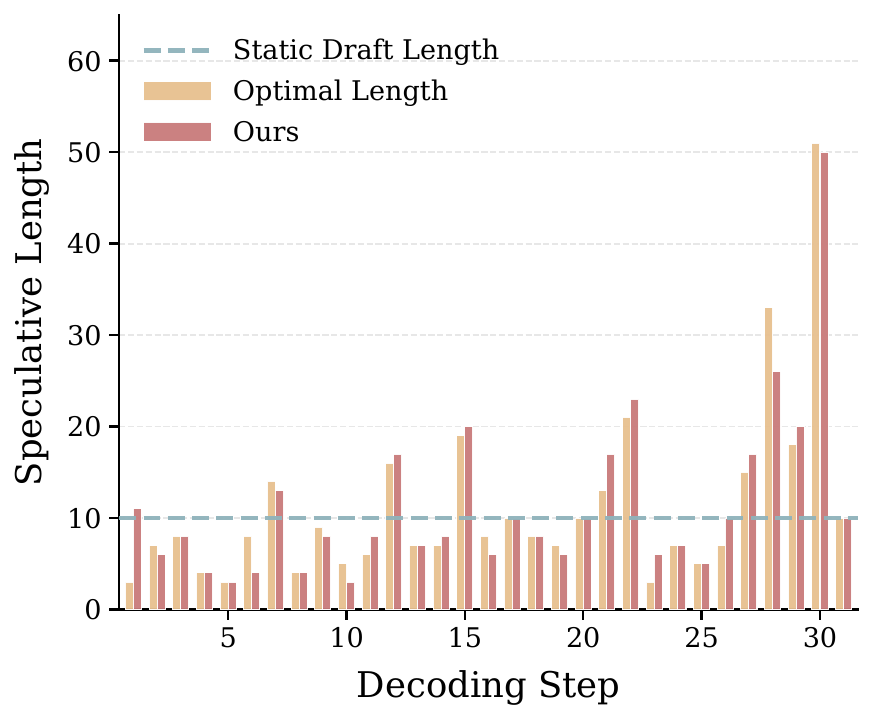}
        \caption{MT-Bench}
        \label{fig:sub4}
    \end{subfigure}

    \caption{Case Study of Dynamic Speculative Length Adjustment with LibraSpec across Four Benchmarks: a FastdLLM Draft Model and a Qwen2.5-7B Target Model.}
    \label{fig:case_study}
\end{figure}

\subsection{Case study of dynamic speculative length adjustment}

In this section, to examine how LibraSpec adapts the speculative length during decoding, Figure \ref{fig:case_study} visualizes decoding trajectories from four benchmarks, comparing three schedules: a static speculative length, the oracle optimal length obtained by exhaustive search, and the length selected by our LibraSpec.

The oracle speculative length varies substantially across decoding steps and depends strongly on the local generation context. This confirms the premise of Section \ref{sc:speedup} — the optimal gain–cost trade-off drifts with local content predictability, so any static length inevitably alternates between under-speculation and over-speculation. Despite having no access to the oracle schedule, LibraSpec closely tracks it using only the confidence and cost signals available online. Across all four examples, LibraSpec follows the oracle schedule with a mean absolute deviation of only 2.78 tokens — far below the 9.48 deviation tokens of the static length schedule — correctly anticipating both expansions and contractions. This is a direct consequence of Theorem \ref{thm:unimodal}: since the expected speedup is unimodal in the speculative length, each beneficial adjustment provably moves toward the current optimum. As the decoding context changes, repeatedly applying such adjustments allows LibraSpec to track the evolving oracle length without explicitly predicting the accepted length.

\section{Conclusion}

In this paper, we show that the recent shift toward diffusion-based drafters changes the optimization objective of dynamic speculative-length selection.
%In this paper, we reveal a fundamental shift in the primary bottleneck of speculative decoding when transitioning to diffusion-based draft models: 
%because parallel block generation renders drafting nearly cost-free, the critical challenge shifts from drafting overhead to verification efficiency. 
Because diffusion drafters generate candidate blocks in parallel at substantially lower marginal drafting cost, estimating the accepted length is no longer sufficient for maximizing end-to-end speedup.
We therefore formulate speculative-length selection as expected-speedup optimization and derive a marginal criterion that balances acceptance gain against verification cost.
%To address this, we establish a rigorous theoretical speedup framework and propose LibraSpec, the first theoretically grounded dynamic length speculative decoding strategy that explicitly balances the expected acceptance gain against verification costs. 
Based on this criterion, we develop \textsc{LibraSpec}, a training-free and plug-and-play algorithm that dynamically determines how many draft tokens should be verified.
%Characterized by its training-free and plug-and-play nature, LibraSpec can be seamlessly integrated into existing diffusion-based methods like FastdLLM and DFlash. 
We show that beneficial adjustments monotonically improve expected speedup and converge to a globally optimal speculative-length interval, while practical \textsc{LibraSpec} approximates these adjustments using drafter confidence scores.
Extensive experiments across mathematical reasoning, code generation, and general chat benchmarks demonstrate that \textsc{LibraSpec} consistently improves verification efficiency under both greedy and sampling settings, yielding a consistent and substantial end-to-end improvement over SOTA methods.

\bibliographystyle{ACM-Reference-Format}
\bibliography{sample-base}

\newpage
\appendix

\section{Theoretical Proofs}

\subsection{Proof for Theorem \ref{thm:gain-cost}}
\label{proof:gain-cost}

\textbf{Theorem \ref{thm:gain-cost}.} A speculative length adjustment achieves a single beneficial adjustment if and only if the marginal benefit of the adjusted draft segment is more favorable than the current average benefit.
    
\begin{equation}
    \begin{cases}
        \displaystyle
        \frac{\mathbb{E}[\tau_{d:d'}]}{T^{\mathrm{verify}}_{d:d'}}
        >
        \frac{\mathbb{E}[\tau_d]}{T^{\mathrm{verify}}_d},
        & d<d', \\[1.5em]
        \displaystyle
        \frac{\mathbb{E}[\tau_{d':d}]}{T^{\mathrm{verify}}_{d':d}}
        <
        \frac{\mathbb{E}[\tau_d]}{T^{\mathrm{verify}}_d},
        & d>d'.
    \end{cases}
\end{equation}

where the first case ($d<d'$) corresponds to extending the draft sequence, requiring the marginal benefit of the appended draft segment to exceed the current average benefit, while the second case ($d>d'$) corresponds to truncating the draft sequence, requiring the marginal benefit of the removed suffix to be lower than the current average benefit.

\begin{proof}
    Substituting the definition of speedup from Definition~\ref{def:speedup}, we obtain: 
    
    \begin{equation}
        \begin{aligned}
            &\mathbb E[\eta_{d'}] > \mathbb E[\eta_d] \\[0.5em]
            &\begin{cases}
                \displaystyle
                \frac{(\mathbb{E}[\tau_d]+\mathbb{E}[\tau_{d:d'}])\, L_{\mathrm{target}}}{T^{\mathrm{verify}}_d+T^{\mathrm{verify}}_{d:d'}}
                >
                \frac{\mathbb{E}[\tau_d]\, L_{\mathrm{target}}}{T^{\mathrm{verify}}_d},
                & d<d', \\[1.5em]
                \displaystyle
                \frac{(\mathbb{E}[\tau_d]-\mathbb{E}[\tau_{d':d}])\, L_{\mathrm{target}}}{T^{\mathrm{verify}}_d-T^{\mathrm{verify}}_{d':d}}
                <
                \frac{\mathbb{E}[\tau_d]\, L_{\mathrm{target}}}{T^{\mathrm{verify}}_d},
                & d>d'.
            \end{cases} \\[0.5em]
            &\begin{cases}
                \displaystyle
                \frac{\mathbb{E}[\tau_d]+\mathbb{E}[\tau_{d:d'}]}{T^{\mathrm{verify}}_d+T^{\mathrm{verify}}_{d:d'}}
                >
                \frac{\mathbb{E}[\tau_d]}{T^{\mathrm{verify}}_d},
                & d<d', \\[1.5em]
                \displaystyle
                \frac{\mathbb{E}[\tau_d]-\mathbb{E}[\tau_{d':d}]}{T^{\mathrm{verify}}_d-T^{\mathrm{verify}}_{d':d}}
                <
                \frac{\mathbb{E}[\tau_d]}{T^{\mathrm{verify}}_d},
                & d>d'.
            \end{cases}
        \end{aligned}
    \end{equation}

    Using the ratio comparison identity, we obtain:

    \begin{equation}
        \begin{cases}
            \displaystyle
            \frac{\mathbb{E}[\tau_{d:d'}]}{T^{\mathrm{verify}}_{d:d'}}
            >
            \frac{\mathbb{E}[\tau_d]}{T^{\mathrm{verify}}_d},
            & d<d', \\[1.5em]
            \displaystyle
            \frac{\mathbb{E}[\tau_{d':d}]}{T^{\mathrm{verify}}_{d':d}}
            <
            \frac{\mathbb{E}[\tau_d]}{T^{\mathrm{verify}}_d},
            & d>d'.
        \end{cases}
    \end{equation}
\end{proof}

\subsection{Proof for Theorem \ref{thm:necessary-condition}}
\label{proof:necessary-condition}

\textbf{Theorem \ref{thm:necessary-condition}}. Under the strengthened position-wise criterion, A Single Beneficial Adjustment from $d$ to $d'$ \textbf{only if} it satisfies:

\begin{equation}
    \begin{aligned}
        d' &\le \epsilon_i, \quad \forall\, i \in [1, d], \\
        \epsilon_i &= \max\left\{d' \in \mathbb Z \mid d' < \frac{T_i^{\mathrm{verify}}}{\tau_i\,c} \, \sum_{j=i}^{d'}\prod_{k=i}^{j}p_k + i\right\}.
    \end{aligned}
\end{equation}

\begin{proof}
    Let $A_{i:j}=\{x_i,\ldots,x_j\text{ are all accepted}\},\ i\le j.$ Since acceptance must occur consecutively,
	
	\begin{equation}
		\Pr(A_{i:j}) = \prod_{k=i}^{j}p_k.
	\end{equation}

    The expected marginal increase in accepted draft tokens is given by
	
	\begin{equation}
		\mathbb E[\tau_{i:j}]
		=
		\sum_{k=i}^{j}\Pr(A_{i:k})
		=
		\sum_{k=i}^{j}\prod_{l=i}^{k}p_l,
	\end{equation}

    where the expectation follows from the linearity of expectation. Retaining the draft token at position $i$ incurs a wasted verification cost of

    \begin{equation}
        T_{i:j} = c\,(j-i).
    \end{equation}

    where $c$ denotes the average verification cost per speculative position. We strengthen the criterion of a single beneficial adjustment by requiring that the marginal benefit of the speculative suffix starting from position i exceeds the average benefit accumulated up to that position.

    \begin{equation}
        \frac{\mathbb{E}[\tau_{i:d'}]}{T^{\mathrm{verify}}_{i:d'}}
        >
        \frac{\tau_i}{T^{\mathrm{verify}}_i}, \, \forall\, i \in [1, d]. 
    \end{equation}

    Expanding the marginal gain and the marginal cost yields

    \begin{equation}
        \begin{aligned}
            \frac{\sum_{j=i}^{d'}\prod_{k=i}^{j}p_k}{c\,(d' - i)} &> 
            \frac{\tau_i}{T^{\mathrm{verify}}_i} \\
            T^{\mathrm{verify}}_i \, \sum_{j=i}^{d'}\prod_{k=i}^{j}p_k &> \tau_i\,c\,(d' - i) \\
            \frac{T_i^{\mathrm{verify}}}{\tau_i\,c} \, \sum_{j=i}^{d'}\prod_{k=i}^{j}p_k + i &> d'. \\
        \end{aligned}
    \end{equation}

    For a single beneficial adjustment, the condition must hold for every draft position $i$. Therefore,

    \begin{equation}
        \begin{aligned}
            d' &\le \epsilon_i, \quad \forall\, i \in [1, d], \\
            \epsilon_i &= \max\left\{d' \in \mathbb Z \mid d' < \frac{T_i^{\mathrm{verify}}}{\tau_i\,c} \, \sum_{j=i}^{d'}\prod_{k=i}^{j}p_k + i\right\}.
        \end{aligned}
    \end{equation}
\end{proof}

\subsection{Proof for Theorem \ref{thm:sufficient-condition}}
\label{proof:sufficient-condition}

\textbf{Theorem \ref{thm:sufficient-condition}}. A Single Beneficial Adjustment from $d$ to $d'$ \textbf{if} it satisfies:

\begin{equation}
    \begin{aligned}
        d' &= \min_i \epsilon_i, \quad \forall\, i \in [1, d], \\
        \epsilon_i &= \max\left\{d' \in \mathbb Z \mid d' < \frac{T_i^{\mathrm{verify}}}{\tau_i\,c} \, \sum_{j=i}^{d'}\prod_{k=i}^{j}p_k + i\right\}.
    \end{aligned}
\end{equation}

\begin{proof}
    \begin{equation}
        d' = \min_i \epsilon_i
    \end{equation}
    
    According to Theorem~\ref{thm:necessary-condition}, since $d'$ must satisfy $d' \le \epsilon_i$ for all $i$, it follows that

    \begin{equation}
        \begin{aligned}
            d' &\le \epsilon_i, \quad \forall\, i \in [1, d], \\
            \epsilon_i &= \max\left\{d' \in \mathbb Z \mid d' < \frac{T_i^{\mathrm{verify}}}{\tau_i\,c} \, \sum_{j=i}^{d'}\prod_{k=i}^{j}p_k + i\right\}.
        \end{aligned}
    \end{equation}

    For $\forall i \in [1,d]$,

    \begin{equation}
        \begin{aligned}
            d' &< \frac{T_i^{\mathrm{verify}}}{\tau_i\,c}  \, \sum_{j=i}^{d'}\prod_{k=i}^{j}p_k + i \\
            \frac{T_i^{\mathrm{verify}}}{\tau_i} &> \frac{c\,(d'-i)}{\sum_{j=i}^{d'}\prod_{k=i}^{j}p_k} \\
            \frac{\mathbb{E}[\tau_{i:d'}]}{T^{\mathrm{verify}}_{i:d'}} &> \frac{\tau_i}{T^{\mathrm{verify}}_i}. 
        \end{aligned}
    \end{equation}

    In particular, substituting $i = d$, 

    \begin{equation}
        \begin{cases}
            \displaystyle
            \frac{\mathbb{E}[\tau_{d:d'}]}{T^{\mathrm{verify}}_{d:d'}}
            >
            \frac{\mathbb{E}[\tau_d]}{T^{\mathrm{verify}}_d},
            & d<d', \\[1.5em]
            \displaystyle
            \frac{\mathbb{E}[\tau_{d':d}]}{T^{\mathrm{verify}}_{d':d}}
            <
            \frac{\mathbb{E}[\tau_d]}{T^{\mathrm{verify}}_d},
            & d>d'.
        \end{cases}
    \end{equation}

    In this case, a single beneficial adjustment can be achieved.
\end{proof}

\subsection{Proof for Theorem \ref{thm:unimodal}}
\label{proof:unimodal}

\textbf{Theorem \ref{thm:unimodal}}. The expected speedup is a unimodal function of the speculative length.

\begin{equation}
    \frac{\mathbb{E}[\tau_{d+1}]}{T^{\mathrm{verify}}_{d+1}}
    >
    \frac{\mathbb{E}[\tau_d]}{T^{\mathrm{verify}}_d} 
    \Longleftrightarrow 
    \prod_{i = 1}^{d+1}p_i
    >
    c\,\frac{\mathbb{E}[\tau_d]}{T^{\mathrm{verify}}_d}
\end{equation}

\begin{proof}
    Since $0\le p_i\le 1$, the product $\prod_{i = 1}^{d}p_i$ is monotonically non-increasing with respect to $d$.

    If the marginal verification cost incurred by adding each speculative position is a constant $c$, then there exists a threshold $d$ such that:

    \begin{equation}
        \frac{\prod_{i = 1}^{d+1}p_i}{c} > \frac{\mathbb{E}[\tau_d]}{T^{\mathrm{verify}}_d}
    \end{equation}

    By Theorem~\ref{thm:gain-cost}, it follows that

    \begin{equation}
        \frac{\mathbb{E}[\tau_{d+1}]}{T^{\mathrm{verify}}_{d+1}}
        >
        \frac{\mathbb{E}[\tau_d]}{T^{\mathrm{verify}}_d} 
        \Longleftrightarrow
        \prod_{i = 1}^{d+1}p_i
        >
        c\,\frac{\mathbb{E}[\tau_d]}{T^{\mathrm{verify}}_d}.
    \end{equation}

    Moreover, its contrapositive also holds.

    \begin{equation}
        \frac{\mathbb{E}[\tau_{d+1}]}{T^{\mathrm{verify}}_{d+1}}
        \le
        \frac{\mathbb{E}[\tau_d]}{T^{\mathrm{verify}}_d} 
        \Longleftrightarrow
        \prod_{i = 1}^{d+1}p_i
        \le
        c\,\frac{\mathbb{E}[\tau_d]}{T^{\mathrm{verify}}_d}
    \end{equation}

    This immediately implies that the expected decoding speedup is a unimodal function of the speculative length.
\end{proof}

\subsection{Proof for Theorem \ref{thm:convergence}}
\label{proof:convergence}

\textbf{Theorem \ref{thm:convergence}}. By iteratively applying single beneficial adjustments, the speculative length converges in finitely many steps to $d'$, which lies in the globally optimal interval of speculative lengths.

\begin{equation}
    \mathbb E[\eta_{d'}]
    =
    \max_{d\in\mathbb Z_{>0}}\mathbb E[\eta_{d}].
\end{equation}

\begin{proof}
    Choose

    \begin{equation}
        d^\star\in
    	\operatorname*{arg\,max}_{d\in\mathbb Z_{>0}}
    	\mathbb E[\eta_d],
    \end{equation}

    Let $\{d^{(k)}\}_{k\ge0}$ denote the sequence of beneficial adjustments. By Definition~\ref{def:speedup},

    \begin{equation}
        \mathbb E[\eta_{d_{k+1}}]
    	>
    	\mathbb E[\eta_{d_k}]
    	\Longrightarrow
    	\mathbb E[\eta_{d_\ell}]
    	>
    	\mathbb E[\eta_{d_k}],
    	\qquad \forall\,\ell>k.
    \end{equation}

    By Theorem~\ref{thm:unimodal},

    \begin{equation}    
    	\begin{aligned}
    		k<\ell,\quad d_k,d_\ell<d^\star
    		&\Longrightarrow d_k<d_\ell,\\
    		k<\ell,\quad d_k,d_\ell>d^\star
    		&\Longrightarrow d_k>d_\ell.
    	\end{aligned}
    \end{equation}

    Let, 

    \begin{equation}
        I_-=\{k:d_k<d^\star\},
    	\qquad
    	I_+=\{k:d_k>d^\star\}.
    \end{equation}

    Then

    \begin{equation}
        |I_-|\le d^\star-1,
    	\qquad
    	I_+\neq\varnothing
    	\Longrightarrow
    	|I_+|
    	\le d_{\min I_+}-d^\star
    	<\infty.
    \end{equation}

    Hence, the sequence reaches the globally optimal interval after
	finitely many beneficial adjustments. At the terminal length $d'$,

    \begin{equation}
        \nexists\,d\in\mathbb Z_{>0}:\mathbb E[\eta_d]>\mathbb E[\eta_{d'}].
    \end{equation}

    Therefore,
    
    \begin{equation}
        \mathbb E[\eta_{d'}]
        =
        \max_{d\in\mathbb Z_{>0}}\mathbb E[\eta_{d}].
    \end{equation}
\end{proof}

\section{Baseline and Benchmark Details}
\label{sc:config-details}

\subsection{Benchmark Details}

Table \ref{tab:benchmark_example_count} lists the number of evaluated examples for each dataset. We follow the DFlash \citep{dflash-chen} benchmark setup for these sample counts.

\begin{table}[htbp]
	\centering
	\caption{Number of evaluated examples per dataset in the benchmark suite.}
	\label{tab:benchmark_example_count}
%	\resizebox{\textwidth}{!}{
		\begin{tabular}{lc}
			\toprule
			Dataset & Examples \\
			\midrule
			MATH-500 \citep{math500-hunter} & 128 \\
			GSM8K \citep{gsm8k-cobbe} & 128 \\
			HumanEval \citep{humaneval-chen} & 164 \\
			MT-Bench \citep{mtbench-zheng} & 80 \\
			\bottomrule
		\end{tabular}
%	}
\end{table}

\subsection{Baseline Details}

\textbf{EAGLE-3}. All EAGLE-3 baselines are run using vLLM v0.13.0 \citep{vllm-kwon}. We run an extensive parameter sweep of speculative length from 3 to 20, and pick a different speculation length that yielded the best speedup for each model-dataset pair. For comparisons with EAGLE-3 on Qwen2.5 models, we use the checkpoints released by FailFast \citep{failfast-pan}.

\textbf{FastdLLM}. All FastdLLM baselines are run using Fast-dLLM-v2-1.5B \citep{fastdllmv2-wu}. we also conduct an extensive parameter sweep of speculative length from 3 to 20, and pick a different speculative length that achieves the best speedup for each model-dataset pair. When combined with LibraSpec, the initial speculative length is set to 10, while the maximum speculative length is set to 60.

\textbf{FailFast}. For the hyperparameters in FailFast, we adopt the optimal configurations reported in the original paper. Specifically, we set $\tau=\{0.5,0.45,0.4\}$ for the Qwen2.5-\{7,14,32\}B-Instruct target models, respectively. The maximum speculative length is set to 80, following the original paper, while the initial speculative length is set to 10.

\textbf{G4-style}. We adapt the dynamic speculative-length control strategy used in Gemma 4 MTP to the FastdLLM. Following the original heuristic, the speculative length for the next decoding round is increased by 2 if all drafted tokens are accepted in the current round, and decreased by 1 whenever at least one drafted token is rejected.

\textbf{DFlash}. All DFlash experiments use a default initial speculative length (block size) of 16. When integrated with LibraSpec, the maximum speculative length is set to 24.

\textbf{DDTree}. All DDTree experiments also use an initial block size of 16. We search the node budget over $\{64, 128, 256, 512\}$ and report the configuration that achieves the highest decoding speedup. When integrated with LibraSpec, the maximum speculative length is likewise set to 24.

\end{document}